\documentclass[accepted]{arxiv} % for initial submission
\usepackage[american]{babel}
\usepackage{natbib} % has a nice set of citation styles and commands
\usepackage{mathtools} % amsmath with fixes and additions
\usepackage{booktabs} % commands to create good-looking tables
\usepackage{tikz} % nice language for creating drawings and diagrams

\usepackage{longtable}
\usepackage{algorithm}
\allowdisplaybreaks

\usepackage[utf8]{inputenc} % allow utf-8 input
\usepackage[T1]{fontenc}    % use 8-bit T1 fonts
\usepackage{hyperref}       % hyperlinks
\usepackage{url}            % simple URL typesetting
\usepackage{booktabs}       % professional-quality tables
\usepackage{amsfonts}       % blackboard math symbols
\usepackage{nicefrac}       % compact symbols for 1/2, etc.
\usepackage{microtype}      % microtypography
\usepackage{xspace}

\usepackage{graphicx, setspace, latexsym ,amsmath,amssymb,color}
\usepackage{caption}
\usepackage{subcaption}
\usepackage{amssymb} %maths
\usepackage{amsmath} %maths
\usepackage{amsthm}
\usepackage{quoting}
\usepackage{dsfont}
\usepackage{bbm}
\usepackage{longtable}
\usepackage{hyperref}
\usepackage{lineno}
\usepackage{enumitem}
\usepackage{wrapfig}

\usepackage{subcaption}
\usepackage{bbm}
\usepackage[noend]{algpseudocode}
\hypersetup{
    colorlinks=true,
    linkcolor=blue,
    filecolor=magenta,      
    urlcolor=cyan,
}

\usepackage{appendix}

\AtBeginEnvironment{subappendices}{%
\chapter*{Appendices}
\addcontentsline{toc}{chapter}{Appendices}
\counterwithin{figure}{section}
\counterwithin{table}{section}
}

\def \P{\mathbb P}
\def \R{\mathbb R}

\def \bx{\boldsymbol{x}}

\def \bq{\boldsymbol{q}}

\def \bQ{\boldsymbol{Q}}

\def \b1{\boldsymbol{1}}

\def \st2{\texttt{\texttt{$($ST$)^2$}}}
\def \T{{\mathcal T}}

\def \X{{\mathcal X}}
\def \M{{\mathcal M}}

\def \S{{\mathcal S}}
\def \hd{\hat{d}}

\def \st2{\texttt{\texttt{$($ST$)^2$}}}

\definecolor{burgundy}{rgb}{0.5, 0.0, 0.13}

\newcommand{\query}[2]{\mathsf{Q}(#1, #2)}
\def \1{\mathbbm{1}}
\def \0{\textbf{0}}

\definecolor{darkorchid}{rgb}{0.6, 0.2, 0.8}

\newtheorem{theorem}{Theorem}

\newtheorem{corollary}[theorem]{Corollary}
\newtheorem{definition}[theorem]{Definition}

\newtheorem{lemma}[theorem]{Lemma}

\newtheorem{remark}[theorem]{Remark}

\title{Nearest Neighbor Search Under Uncertainty}

\author[1]{Blake Mason}
\author[2]{Ardhendu Tripathy}
\author[1]{Robert Nowak}
\affil[1]{%
    Department of Electrical and Computer Engineering.\\
    University of Wisconsin\\
    Madison, Wisconsin, USA
}
\affil[2]{%
    Computer Science Department.\\
    Missouri University of Science and Technology\\
    Rolla, Missouri, USA
}

\begin{document}
\maketitle

\begin{abstract}
Nearest Neighbor Search (NNS) is a central task in knowledge representation, learning, and reasoning.  There is vast literature on efficient algorithms for constructing data structures and performing exact and approximate NNS.  This paper studies NNS under Uncertainty (NNSU).  Specifically, consider the setting in which an NNS algorithm has access only to a stochastic distance oracle that provides a noisy, unbiased estimate of the distance between any pair of points, rather than the exact distance.  This models many situations of practical importance, including NNS based on human similarity judgements, physical measurements, or fast, randomized approximations to exact distances.  A naive approach to NNSU could employ any standard NNS algorithm and repeatedly query and average results from the stochastic oracle (to reduce noise) whenever it needs a pairwise distance.  The problem is that a sufficient number of repeated queries is unknown in advance; e.g., a point may be distant from all but one other point (crude distance estimates suffice) or it may be close to a large number of other points (accurate estimates are necessary).  This paper shows how ideas from cover trees and multi-armed bandits can be leveraged to develop an NNSU algorithm that has optimal dependence on the dataset size and the (unknown) geometry of the dataset.
\end{abstract}

\section{Introduction}\label{sec:intro}

This paper considers Nearest Neighbor Search under Uncertainty (NNSU).  
% \rob{\noindent\fbox{
% \parbox{0.45\textwidth}{\textbf{Problem Statement:}
% Consider a set of $n$ points $\X = \{\bx_1, \cdots, \bx_n\}$ in a metric space $(\M, d)$. The metric is unknown, but for any pair of points we can query a stochastic oracle for a noisy, unbiased estimate of their distance. The problem is to use the oracle to efficiently build a data structure such that for any new query point $\bq$, it returns its nearest neighbor $\bx_q := \min_{\bx_i \in \X} d(\bq, \bx_i)$ with probability at least $1-\delta$ in as few additional oracle queries as possible. 
% }}}
To motivate the NNSU problem, consider the following application.  Suppose we have a database of $N$ genomic sequences of length $L$ of different species and wish to query the database to find the closest relative of a newly discovered species.  Computing the exact distance in $\R^L$ between two sequences requires $O(L)$ operations. Using efficient NNS algorithms, one can construct a data structure in $O(NL\log N)$ time and perform a NN search in $O(L\log N)$. To reduce this complexity, one can randomly subsample the sequences at $\ell\ll L$ locations and compute an unbiased estimate of distance in $O(\ell)$ time and find nearest neighbors in $O(\ell\log(n))$ operations. If an algorithm can manage the added uncertainty of this procedure, this can improve computational complexity greatly. 
% This paper shows that it is possible to construct a data struc-ture inO(NlogNlogL)time and with large probability findthe nearest neighbor of any new point inO(logNlogL)time

In other settings, uncertainty is naturally present in the measurements due to the presence of noise. For instance, researchers gathering data to map the topology of the internet, rely on noisy one-way-delay measurements to infer distance between servers \cite{eriksson2010learning} and, detecting the closest server to a given host can be challenging due to this noise. In preference learning, researchers gather pairwise comparisons from people 
and attempt to learn which items are most preferred \cite{jamieson2011active, yue2009interactively}. 
To simplify such problems, human judgments are frequently modelled as distance comparisons \cite{shepard1962analysis, kruskal1964nonmetric, mason2017learning} with the smallest distance representing the most preferred item,
but the noise inherent to human judgements makes these problems challenging. In general, we define the NNSU problem as follows:

\noindent\fbox{
\parbox{0.45\textwidth}{\textbf{NNSU - Problem Statement:}
Consider a set of $n$ points $\X = \{\bx_1, \cdots, \bx_n\}$ in a metric space $(\M, d)$. The metric is unknown, but for any pair of points we can query a stochastic oracle for a noisy, unbiased estimate of their distance. The problem is to use the oracle to efficiently build a data structure such that for any new query point $\bq$, it returns its nearest neighbor $\bx_q := \min_{\bx_i \in \X} d(\bq, \bx_i)$ with probability at least $1-\delta$ in as few additional oracle queries as possible. 
}}

% The NNSU problem is very different from the standard Nearest Neighbor Search (NNS) problem. To develop noise-tolerant methods, it is tempting to simply extend an NNS algorithm by repeatedly calling the stochastic oracle a \emph{pre-specified} $\ell$ times each time a distance measurement is needed and averaging the results. 
% The simplicity of this idea masks the difficulty of doing this correctly. 
% % If the specified parameter is too small, an algorithm will make errors, but if the parameter is too large, computational cost will be unnecessarily high.
% If $\ell$ is too small, an algorithm will make errors, but if $\ell$ is too large, the number of distance queries (statistical cost) and computational cost will be unnecessarily high. To control the probability of error, $\ell$ can be no smaller than $O((d(\bq, \bx_q) - d(\bq, \bx_{q'}))^{-2})$ where $\bx_{q'}$ is the second nearest neighbor to $\bq$. Naturally, as this quantity depends on knowledge of the nearest neighbor itself, it is unreasonable to require such an $\ell$ as input to the algorithm. Additionally, even if this value were known, this is worst case and may be far more calls to the stochastic oracle than are necessary to eliminate points that are clearly not $\bq$'s nearest neighbor. Instead, our algorithm that is \emph{adaptive} to the amount of noise and the geometry of $\X$ and minimizes the number of calls to the distance oracle for a specified probability of error. 

The NNSU problem is very different from the standard Nearest Neighbor Search (NNS) problem. To develop noise-tolerant methods, it is tempting to simply extend an NNS algorithm by repeatedly calling the stochastic oracle a \emph{pre-specified} $r$ times each time a distance measurement is needed and averaging the results. 
The simplicity of this idea masks the difficulty of doing it correctly. 
% If the specified parameter is too small, an algorithm will make errors, but if the parameter is too large, computational cost will be unnecessarily high.
If $r$ is too small, an algorithm will make errors, but if $r$ is too large, the number of distance queries and hence the computational cost will be unnecessarily high. To control the probability of error\footnote{We assume the oracle responses are independent and $1$-subGaussian distributed throughout the paper.}, $r$ can be no smaller than $(d(\bq, \bx_q) - d(\bq, \bx_{q'}))^{-2}$ where $\bx_{q'}$ is the second nearest neighbor to $\bq$. Since this quantity depends on knowledge of the nearest neighbor distances, it is impossible to specify $r$ in practice. Additionally, even if $(d(\bq, \bx_q) - d(\bq, \bx_{q'}))^{-2}$ were known, setting $r$ in proportion to it is only necessary to handle the worst case; far fewer calls to the stochastic oracle are sufficient to eliminate points that are far from $\bq$. Our proposed algorithm is \emph{adaptive} to the level of noise and the geometry of $\X$. It is guaranteed to minimize the number of calls to the distance oracle for a specified probability of error. 

\subsection{Related Work}\label{sec:nns_related}

One of the first algorithms for the NNS problem is the classical \texttt{KD-Tree} algorithm by \cite{bentley1975multidimensional}. \texttt{KD-Tree}s first build a tree-based data structure in $O(n\log(n))$ computations. By paying this cost up front, when presented with a query point, the tree can be used to return the nearest neighbor in only $O(\log(n))$ computations. The core drawback of the \texttt{KD-Tree} is that it lacks a uniform accuracy guarantee and may fail to return the correct nearest neighbor for certain query points \cite{dasgupta2013randomized}. 
% Given access to \textit{noiseless} distance measurements or other oracles, 
A variety of methods attempt to achieve similar computational complexity for the NNS problem while also providing robust accuracy guarantees \cite{dasgupta2013randomized, krauthgamer2004navigating,beygelzimer2006cover}, and we refer the reader to \cite{bhatia2010survey} for a survey of techniques and theoretical results. Another line of work attempts to instead return an approximate nearest neighbor that is almost as close as the true nearest neighbor. We refer the reader to \cite{wang2014randomized, andoni2018approximate}  for an overview. 

A related problem to NNSU is the noisy nearest neighbor graph problem. In this setting, one wishes to learn the graph that connects each node in $\X$ to its nearest neighbor.
\cite{mason2019learning} provide an adaptive method that utilizes multi-armed bandits to find the nearest neighbor graph from a set of $n$ points in $O(n\log(n)\Delta^2)$ samples in favorable settings and $O(n^2\Delta^2)$ at worst where $\Delta^2$ is a problem dependent parameter quantifying the effect of the noise. Their method is adaptive to noise, but requires additional assumptions to achieve the optimal rate of $O(n\log(n))$ samples for that problem. 
\cite{bagaria2017medoids} study nearest neighbor search in high-dimensional data. In their setting, the distance function is a known, componentwise function. 
Their algorithm subsamples to approximate distances. This improves dependence on dimension, though the dependence on $n$ is linear. 

\subsection{Main Contributions}
In this paper, we leverage recent developments in multi-armed bandits to solve the nearest neighbor search problem using noisy measurements.  
We design a novel extension of the \texttt{Cover-Tree} algorithm from \cite{beygelzimer2006cover} for the NNSU problem. A main innovation is reducing the problem of learning a cover of a set 
from noisy data to all-$\epsilon$-good identification for multi-armed bandits studied in \cite{mason2020finding}. 
Additionally, we make use of efficient methods for adaptive hypothesis testing to minimize calls to the stochastic oracle \cite{jamieson2018bandit}. 
We refer to the resulting algorithm as the \texttt{Bandit-Cover-Tree}. We show that it requires $O(n\log^2(n)\kappa)$ calls to the distance oracle for construction and $O(\log(n)\kappa)$ calls for querying the nearest neighbor of a new point, where $\kappa$ captures the effect of noise and the geometry of $\X$. This nearly matches the state of the art for the NNS problem despite the added challenge of uncertainty in the NNSU problem. 
% Furthermore, we show that in the setting of \cite{bagaria2017medoids} where one uses Monte-Carlo optimization to accelerate distance computation, \texttt{Bandit-Cover-Tree}s can be used to return the nearest neighbor in $O(\log(n)\log(d))$ samples, an exponential improvement over the $O(n\log(d))$ rate from \cite{bagaria2017medoids}. 
Furthermore, we show how to extend \texttt{Bandit-Cover-Tree} for approximate nearest neighbor search instead. Finally, we demonstrate that \texttt{Bandit-Cover-Tree} can be used to learn a nearest neighbor graph in $O(n\log^2(n)\kappa)$ distance measurements, nearly matching the optimal rate given in \cite{mason2019learning} but without requiring the additional assumptions from that work.

\subsection{Notation}\label{sec:setup}
Let $(\M, d)$ be a metric space with distance function $d$ satisfying the standard axioms. Given $\bx_i,\bx_j\in \X$, let $d(\bx_i, \bx_j) = d_{i,j}$. For a query point $\bq$ define $\bx_q := \arg\min_{\bx \in \X }d(x_i, \bq)$. For a query point $\bq$ and $\bx_i \in \X$, let $d_{q,i}$ denote $d(\bq, \bx_i)$. Additionally, for a set $\S$, define the distance from any point $\bx \in (\M, d)$ to $\S$ as $d(\bx, \S) := \inf_{z\in S}d_{\bx,z}$, the smallest distance a point in $\S$.  
Though the distances are unknown, we are able to draw independent samples of its true value according to a stochastic distance oracle, i.e.\ querying
\begin{equation}\label{eq:nns_oracle}
\query{i}{j}\;\;\; \text{ yields a realization of } \;\;\; d_{i, j} + \eta,
\end{equation}
where $\eta$ 
is a zero-mean subGaussian random variable assumed to have scale parameter $\sigma = 1$. 
We let $\hd_{i,j}(s)$ denote the empirical mean of the  $s$ queries of the distance oracle, $\query{i}{j}$. The number of $\query{i}{j}$ queries made until time $t$ is denoted as $T_{i,j}(t)$. 
% A possible approach to obtain the $\bq$'s nearest neighbor is to repeatedly query $\query{q}{i}$ for each $i$ and report $\arg\min_{\bx \in \X}\hd_{q, i}(t)$. 
All guarantees will be shown to hold with probability $1-\delta$ where we refer to $\delta > 0$ as the failure probability. 

% To improve our query efficiency, we could instead adaptively sample to focus queries on distances that we estimate are smaller. However, this would still require $O(n)$ samples to learn $\bx_q$ as we would query the distance of each of $n$ points to $\bq$. Instead, we modify the Cover Tree algorithm of \cite{beygelzimer2006cover} to handle noisy inputs instead. Two core innovations that make this possible are
% \begin{itemize}
% %\setlength\itemsep{-0.5em}
% \item Casting the problem of learning a cover as finding all $\epsilon$-good arms in multiarmed bandits and using \st2 presented in \cite{mason2020finding} to develop a method to learn covers. 
% \item Using thresholding bandits, such as \cite{jamieson2018bandit}, to compute the approximate distance from a query point to a set. 
% \end{itemize}
% Our proposed algorithm \texttt{Bandit Cover Tree} (\texttt{BCT}) uses the above ideas to 
% build a data structure to efficiently answer nearest neighbor queries from noisy distance measurements with high probability. 
\section{Cover Trees for Nearest Neighbor Search}

Before presenting our method and associated results, we the review \texttt{Cover Tree} algorithm from \cite{beygelzimer2006cover} for the NNS problem. As the name suggests, a cover tree is a tree-based data structure where each level of the tree forms a \emph{cover} of $\X$. 
\begin{definition}
Given a set $\X \subset (\M, d)$, a set $C \subset \X$ is a \emph{cover} of $\X$ with resolution $\epsilon$ if for all $\bx \in \X$, there exists a $c \in C$ such that $d(\bx, c)\leq \epsilon$. 
\end{definition}
Each level is indexed by an integer $i$ which decreases as one descends the tree. 
% The nodes in the $\ith$ level of the tree form a cover of resolution $2^i$. Hence as one descends the tree, each level forms a cover of progressively higher resolution. 
To avoid additional notation, we will represent the top of the tree as level $\infty$ and the bottom as $-\infty$ though in practice one would record integers $i_\text{top}$ and $i_{\text{bottom}}$ denoting the top and bottom level of the tree and need only explicitly store the tree between these levels. Each node in the tree corresponds to a point in $\X$, but points in $\X$ may correspond to multiple nodes in the tree. Reviewing \cite{beygelzimer2006cover}, let $C_i$ denote the set of nodes at level $i$. The cover tree algorithm is designed so that each level of the tree $i$ obeys three invariants:

\begin{enumerate}
\item \textbf{nesting:} $C_i \subset C_{i-1}$. Hence, the points corresponding to nodes at level $i$ are also correspond to nodes in all lower levels.
\item \textbf{covering tree:} For every $p \in C_{i-1}$, there exists a $q\in C_i$ such that $d(p, q) \leq 2^i$, and child node $p$ is connected to parent node $q$ in the cover tree. 
\item \textbf{Separation:} For any $p, q \in C_i$ with $p\neq q$, $d(p,q) > 2^i$. 
\end{enumerate}

These invariants are originally derived from \cite{krauthgamer2004navigating}. \cite{beygelzimer2006cover} show that their routines obey these invariants, and we will make use of them in our proofs and to build intuition. 
The core idea of this method is that the $i^\text{th}$ level of the tree is a cover of resolution $2^i$. Given a query point $q$, when navigating the tree at level $i$, one identifies all possible ancestor nodes of $\bx_q$ at that level. When descending the tree, this set is refined until only a single parent is possible, $\bx_q$ itself. 
%Intuitively, each level of the tree gives a mesh with resolution $2^i$. As $i$ gets small when one descends the tree, the only point in a single bucket is $\bq$'s nearest neighbor. 
The \emph{nesting} invariance allows one to easily traverse the tree from parent to child. The \emph{covering tree} invariance connects $\bx_q$ to its parents and ancestors so that one may traverse the tree with query point $\bq$ and end at $\bx_q$. Lastly, the \emph{separation} invariance ensures that there is a single parent to each node which avoids redundancy and improves memory complexity. 

In order to quantify the sample complexity of cover trees, we require a notion of the effective dimensionality of the points $\X$. 
In this paper, we will make use of the \emph{expansion constant} as in \citep{haghiri2017comparison, krauthgamer2004navigating, beygelzimer2006cover}. Let $B(\bx, r)$ denote the ball of radius $r> 0$ centered at $\bx\in (\M, d)$ according to the distance measure $d$.

\begin{definition}\label{def:expansion_constant}
The expansion constant of a set of points $\S$ is the smallest $c \geq 2$ such that $|\S \cap B(\bx, 2r)| \leq c |\S \cap B(\bx, r)| $ for any $\bx \in \S$ and any $r > 0$. 
\end{definition}

The expansion constant is sensitive to the geometry in $\X$. For example, for points in a low-dimensional subspace, $c = O(1)$ independent of the ambient space. By contrast, if points that are spread out on the surface of a sphere in $\R^d$, $c = O(2^d)$. In general $c$ is smaller if the pairwise distances between points in $\X$ are more varied.  
\section{The Bandit Cover Tree Algorithm for NNSU}

The \texttt{Bandit Cover Tree} algorithm is comprised of three methods. Given a cover tree and query point $\bq$, \texttt{Noisy-Find-Nearest} finds $\bq$'s nearest neighbor. \texttt{Noisy-Insert} allows one to insert points into a tree and can be used to construct a new tree. \texttt{Noisy-Remove} can remove points from the tree and is deferred to the appendix. 
%Throughout, we will allow these methods to take in the expansion constant or any bound if one is known and otherwise treat $c=\infty$. 

\subsection{Finding Nearest Neighbors with a Cover Tree}

We begin by discussing how to identify the nearest neighbor of a query point $\bq$ in the set $\X$ given a cover tree $\T$ on $\X$. Throughout, we take $\T$ to denote the cover tree. We may identify $\T$ by the covers at each level: $\T := \{C_{\infty}, \cdots, C_i, C_{i-1}, \cdots, C_{-\infty}\}$. Assume that we are given a fixed query point $\bq$, and the expansion constant of the set $\X \cup \{\bq\}$ is bounded by $c$. Our method to find $\bx_q$ is \texttt{Noisy-Find-Nearest} and is given in Algorithm~\ref{alg:nns_search}. It is inspired by \cite{beygelzimer2006cover}, but a crucial difference is how the algorithm handles noise. At a high level, as it proceeds down each level $i$, it keeps track of a set $Q_i \subset C_i$ of all possible ancestors of $\bx_q$. Given $Q_i$, to proceed to level $i-1$, the algorithm first computes the set of children of nodes in $Q_i$ given by the set $Q \subset C_{i-1}$. $Q_{i-1}\subset Q$ is then computed by forming a \emph{cover} of $Q$ at resolution $2^{i-1}$. Ideally, this is the set 
% $$\left\{i\in Q: d(q, i) \leq \min_{j\in Q}d(q,j) + 2^{i-1}\right\}.$$ 
\begin{align}\label{eq:set_to_find}
    \Tilde{Q}_{i-1} := \left\{j\in Q: d(q, j) \leq \min_{k\in Q}d(q,k) + 2^{i-1}\right\}
\end{align}
However, constructing this set requires calling the stochastic distance oracle and presents 
a trade-off:
\begin{enumerate}
    \item By using many repeated queries to the stochastic oracle, 
    we can more confidently declare if $j \in \Tilde{Q}_{i-1}$ at the expense of higher sample complexity. 
    \item By using fewer calls to the oracle, we can more quickly declare if $j \in \Tilde{Q}_{i-1}$, at the expense of lower accuracy. 
\end{enumerate}

To handle the noise, we make use of anytime confidence widths, $C_{\delta}(t)$ such that for an empirical estimate of the distance $d_{q,j}$: $\hd_{q,j}(t)$ of $t$ samples, we have $\P(\bigcup_{t=1}^\infty |\hd_{q,j}(t) - d_{q,j}| > C_{\delta}(t)) \leq \delta$. For this work, we take $C_\delta(t) = \sqrt{\frac{4\log(\log_2(2t)/\delta)}{t}}$ akin to \cite{howard2018uniform}. This allows us to quantify how certain or uncertain an estimate of any distance is. Furthermore, to guard against settings where the number of samples needed to detect the set $\Tilde{Q}_{i-1}$ in Eq~\eqref{eq:set_to_find} is large, we introduce some slack and instead detect a set $Q_i$ such that it contains every $j\in Q$  such that $d_{q,j} \leq d(\bq, Q) + 2^{i-1}$ and no $j$ worse than $d_{q,j} \leq d(\bq, Q) + 2^{i-1} + 2^{i-2}$. 
% \rob{{\bf definition of $\Tilde{Q}_{i-1}$ uses $i$ to refer to elements included in the set.}} 
In particular, this contains all points in $\Tilde{Q}_{i-1}$ and none that are much worse. This has the advantage of controlling the number of calls to the stochastic oracle. As the algorithm descends the tree, $i$ decreases and the slack of $2^{i-2}$ goes to zero ensuring that the algorithm returns the correct nearest neighbor. To handle both of these challenges, we reduce the problem of learning $\Tilde{Q}_{i-1}$ to the problem of finding all $\epsilon$-good arms in mutli-armed bandits studied in \cite{mason2020finding} and refer to this method as \texttt{Identify-Cover} given in Alg.  \ref{alg:nns_build_cover}.

% and a major challenge of this work is to compute this set correctly with high probability, but in as few samples as possible. Algorithm~\ref{alg:nns_build_cover} shows how we compute this set. Crucially, the problem of computing $Q_i$ can be reduced to the problem of finding all $\epsilon$-good arms in mutli-armed bandits studied in \cite{mason2020finding}. \texttt{Identify-Cover} is similar to the \st2 algorithm studied therein with some modification to account for the fact that we want to find the smallest distances in $Q$. To handle the noisy measurements from the stochastic oracle, the routine maintains anytime confidence widths, $C_{\delta}(t)$ such that for an empirical estimate of the distance $d_{q,i}$: $\hd_{q,i}(t)$ of $t$ samples, we have $\P(\bigcup_{t=1}^\infty |\hd_{q,i}(t) - d_{q,i}| > C_{\delta}(t)) \leq \delta$. For this work, we take $C_\delta(t) = \sqrt{\frac{c_\phi\log(\log_2(2t)/\delta)}{t}}$ for a constant $c_\phi$. It suffices to take $c_\phi = 4$ \cite{howard2018uniform}.

This process of exploring children nodes and computing cover sets terminates when the algorithm reaches level $i = -\infty$ and $Q_{-\infty}$ contains $\bx_q$. Note that in practice the algorithm could terminate when it reaches the bottom of the cover tree, $i = i_\text{bottom}$. At this level, it calls a simple elimination scheme \texttt{find-smallest-in-set} given in the appendix that queries the stochastic oracle and returns $\bx_q \in Q_{-\infty}$. 
To control the overall probability of error in \texttt{Noisy-Find-Nearest}, each call to \texttt{Identify-Cover} and \texttt{find-smallest-in-set} is run with failure probability $\delta/\alpha$ for $\alpha$ chosen to control the family-wise error rate.

% As it proceeds down each level $i$, it keeps track of a set $Q_i \subset C_i$ of all possible ancestors of $\bx_q$, $\bq$'s nearest neighbor. This is summarized in Algorithm~\ref{alg:nns_search} and adapted from \cite{beygelzimer2006cover}. 
% The algorithm proceeds by exploring descending from level $i$ to level $i-1$ and exploring all children of the nodes in $Q_i$. From this, the algorithm recomputes possible parents of $\bx_q$ to form $Q_{i-1}$. 
% The algorithm terminates when it reaches a level $i_\text{bottom}$. Throughout, we will represent the children of any node $p \in \T$ as $\text{children}(p)$. For simplicity, we assume that the nearest neighbor of $\bq$ is unique. We make use of a novel subroutine to build cover sets from noisy distance measurements, based on finding all $\epsilon$-good arms in stochastic bandits. This is given in Algorithm~\ref{alg:nns_build_cover}. The routine maintains anytime confidence widths, $C_{\delta}(t)$ such that for an empirical estimate of the distance $d_{q,i}$: $\hd_{q,i}(t)$ of $t$ samples, we have $\P(\bigcup_{t=1}^\infty |\hd_{q,i}(t) - d_{q,i}| > C_{\delta}(t)) \leq \delta$. For this work, we take $C_\delta(t) = \sqrt{\frac{c_\phi\log(\log_2(2t)/\delta)}{t}}$ for a constant $c_\phi$. It suffices to take $c_\phi = 4$, though tighter bounds are known and should be used in practice, e.g. \cite{jamieson2014lil, kaufmann2016complexity, howard2018uniform}. Where clear, we will drop the dependence on $t$ and $T_i(t)$. 

\begin{algorithm}[tb]
   \caption{\texttt{Noisy-Find-Nearest} \label{alg:nns_search} }
\begin{algorithmic}[1]
\Require{Cover tree $\T$, failure probability $\delta$, expansion constant $c$ if known, query point $\bq$, callable distance oracle $\query{\cdot}{\cdot}$, subroutine \texttt{Identify-Cover}.}
\State{Let $Q_\infty = C_\infty$, $\alpha = n+1$}
\For{$i=\infty $ down to $i = -\infty$}
	\State{Let $Q =\bigcup_{p \in Q_{i}} \text{children}(p) $ }
	\If{$c$ is known:}
		\State{Let $\alpha =\min\left\{\left\lceil\frac{\log(n)}{\log(1 + 1/c^2)} + 1\right\rceil, n\right\}+1$}
	\EndIf
	\State{$\backslash\backslash$ Identify the set: $\{j\in Q: d(q, j) \leq \min_{k\in Q}d(q,k) + 2^{i-1}\}$}
	\State{$Q_{i-1} = \texttt{Identify-Cover}\left(Q, \delta/\alpha, \query{\cdot}{\cdot}, \bq, i\right)$}
\EndFor \\
\Return{\texttt{find-smallest-in-set}($Q_{-\infty}, \delta/\alpha$)}
\end{algorithmic}
\end{algorithm}

\begin{algorithm}[tb]
   \caption{\texttt{Identify-Cover} \label{alg:nns_build_cover}}
\begin{algorithmic}[1]
\Require{Failure probability $\delta$, query point $\bq$, Oracle $\query{\cdot}{\cdot}$, and set $\bQ$, Cover resolution $2^i$}
\State{Query oracle once for each point in $Q$}
\State{Initialize $T_j \leftarrow 1$, update $\hd_{q,j}$ for each $j \in Q$}
\State{Empirical cover set: $\widehat{G} = \{j: \hd_{q,j} \leq \max_k \hd_{q,k} + 2^i\}$}
\State{$U_t = \min_k \hd_{q,k}(T_k) + C_{\delta/|Q|}(T_k) + 2^i$} 
\State{$L_t = \min_k \hd_{q,k}(T_k) - C_{\delta/|Q|}(T_k) + 2^i + 2^{i-1}$}
\State{Known points: $K = \{j: \hd_{q,j}(T_j) - C_{\delta/|Q|}(T_j) > U_t \text{ or } \hd_{q,j}(T_j) + C_{\delta/|Q|}(T_j) < L_t\}$}
\While{$K \neq Q$}
\State{$j_1(t) = \arg\min_{j \in \widehat{G}\backslash K}  \hd_{q,j}(T_j) + C_{\delta/|Q|}(T_j)$}
\State{$j_2(t) = \arg\max_{j \in \widehat{G}^c \backslash K} \hd_{q,j}(T_j) - C_{\delta/|Q|}(T_j)$}
\State{$j^\ast(t) = \arg\min_{j} \hd_{q,j}(T_j) - C_{\delta/|Q|}(T_j)$}
\State{Call oracle $\query{q}{j_1}$, $\query{q}{j_2}$, and $\query{q}{j^\ast}$}
\State{Update $T_{j_1}, \hd_{q, j_1}$, $T_{j_2}, \hd_{q, j_2}$, $T_{j^\ast}, \hd_{q, j^\ast}$}
% \State{Call oracle $\query{q}{i_1}$ for $i_1(t) = \arg\min_{i \in \widehat{G}\backslash K}  \hd_{q,i}(T_i) + C_{\delta/|Q|}(T_i)$, update $T_{i_1}, \hd_{q, i_1}$}
% \State{Call oracle $\query{q}{i_2}$ for  $i_2(t) = \arg\max_{i \in \widehat{G}^c \backslash K} \hd_{q,i}(T_i) - C_{\delta/|Q|}(T_i)$, update $T_{i_2}, \hd_{q, i_2}$}
% \State{Call oracle $\query{q}{i^\ast}$ for  $i^\ast(t) = \arg\min_{i} \hd_{q,i}(T_i) - C_{\delta/|Q|}(T_i)$, update $T_{i^\ast}, \hd_{q, i^\ast}$}
\State{Update bounds $L_t, U_t$, sets $\widehat{G}$, $K$}

%\State{Update bounds and empirical means}
\EndWhile\\
\Return The set cover set with resolution $2^i$: $\{j: \hd_{q,j}(T_j) + C_{\delta/|Q|}(T_j) < L_t\}$
\end{algorithmic}
\end{algorithm}

%In each round as the algorithm descends the tree, it keeps tracks of a set $Q_i \subset C_i$ of all ancestors in level $i$. 
% Given the set $Q_i\subset C_i$ of $\bx_q$'s ancestors in level $i$, to proceed to level $i-1$, the algorithm first computes all children of nodes in $Q_i$ given by the set $Q \subset C_{i-1}$. It then uses Algorithm~\ref{alg:nns_build_cover} to identify the subset of $Q$ denoted $Q_{i-1}$ that contains all ancestors of $\bx_q$ in level $i-1$. Precisely, this is the set $\{i\in Q: d(q, i) \leq \min_{j\in Q}d(q,j) + 2^{i-1}\}$. In particular, this can be represented as the set of all $2^i$-good arms (in the additive sense) in the set $Q$ with the key distinction that we want the \emph{smallest} distances not the largest. This can be achieved by multiplying each distance estimate by $-1$ and finding the $2^i$-good arms.

% \begin{remark}
% This method can be altered so that a point $p \in \X$ that corresponds to a node in $C_i$ at level $i$ of the cover tree $\T$ is defined to always be contained in  $\text{children}(p)$, even if a node corresponding to point $p$ is not present in $C_{i-1}$ (the next level down) explicitly. Traversal is identical since this is baked into the definition of the $\text{children}(\cdot)$ function, but this can save on the memory needed to store $\T$.
% \end{remark} 

\subsubsection{Approximate NNSU}

In some applications, finding an exact nearest neighbor may be unnecessary and an \emph{approximate} nearest neighbor may be sufficient. 
% For instance, in the case of recommendation systems, one may only need to recommend a sufficiently good item to users even if it is not the best. 
In the noiseless case, this has been shown to improve the complexity of nearest neighbor search \cite{andoni2018approximate}. An $\epsilon$-approximate nearest neighbor is defined as any point in the set 
$$\{i: d(q,i) \leq (1+\epsilon)\min_jd(q, j)\}.$$
In this section, we show to how to extend \texttt{Noisy-Find-Nearest} to return an $\epsilon$-approximate nearest neighbor instead of the exact nearest neighbor and provide psuedocode for this method in the appendix. 
To do so, add an additional line that 
exits the for loop if $d(q, Q_i) \geq 2^{i+1}(1 + 1/\epsilon)$.
Then, return $\bx_{q_\epsilon}  =\arg\min_{j \in Q_i}d(\bq, \bx_j)$. To see why this condition suffices, note that the nesting and covering tree invariances jointly imply that $d(\bx_q, Q_i) \leq 2^{i+1}$. Hence, by the triangle inequality
\begin{align*}
    d(\bq, Q_i) &\leq d(\bq, \bx_q) + d(\bx_q, Q_i)\leq d(\bq, \bx_q) + 2^{i+1}.
\end{align*}
Since we exit when $d(q, Q_i) \geq 2^{i+1}(1 + 1/\epsilon)$, 
\begin{align*}
    2^{i+1}(1 + 1/\epsilon) \leq d(\bq, \bx_q) + 2^{i+1}
     \iff 2^{i+1} \leq \epsilon d(\bq, \bx_q).
\end{align*}
Therefore, 
\begin{align*}
    d(\bq, Q_i) \leq d(\bq, \bx_q) + 2^{i+1} \leq (1+\epsilon)d(\bq, \bx_q).
\end{align*}
Hence, there exists an $\epsilon$-approximate nearest neighbor in $Q_i$. If this condition is never satisfied, which occurs when $\epsilon$ is vanishingly small, the algorithm terminates normally and returns the exact nearest neighbor. 
% Exactly how many calls to the distance oracle is a function of $\epsilon$, the query point $\bq$, and the set of points $\X$ themselves. 
We give pseudocode for this procedure in the appendix. The algorithm is similar to \texttt{Noisy-Find-Nearest} except that it makes use of a simple thresholding bandit, akin to the one we use in \texttt{Noisy-Insert} to check if $d(q, Q_i) \geq 2^{i+1}(1 + 1/\epsilon)$ by querying the stochastic oracle. 

\subsection{Building and Altering a Cover Tree}

In this section we demonstrate how to insert points into a cover tree by calling the stochastic oracle. Constructing a cover tree can be achieved by simply beginning with an empty tree and inserting points one at a time. Removing points from a cover tree is similar operationally to insertion with slight complications and we defer it to the appendix. 

\subsubsection{Insertion} \label{sec:nns_insert}
Suppose we have access to a cover tree $\T$ on a set $\S$. 
% If $\S = \emptyset$, $\T$ is empty at all levels. 
We wish to insert a point $p$ into $\T$ such that we now have a cover tree on the set $\S \cup \{p\}$. 
Intuitively, the insertion algorithm can be thought of as beginning at the highest resolution cover, at level $-\infty$ and climbing back up the tree, inserting $p$ in each cover set $C_i$ for all $i$ until it reaches a level $i_p$ such that $\min_{j\in C_{i_p}}d(p, j) \leq 2^{i_p}$ where a suitable parent node exists. The algorithm then chooses a parent $p'\in C_{i_p}$ for $p$ such that $d(p,p') \leq 2^{i_p}$ and terminates. As trees are traditionally traversed via their roots not their leaves, we state this algorithm recursively beginning at the top.

\begin{algorithm}[tb]
   \caption{\texttt{Noisy-Insert} \label{alg:nns_insert} }
\begin{algorithmic}[1]
\Require{Cover tree $\T$ on $n$ points, cover set $Q_i$, failure probability $\delta$, point $p$ to be inserted, callable distance oracle $\query{\cdot}{\cdot}$, level $i$}
\Statex{\hspace{-1.5em}\textbf{Optional: }Empirical estimates of $\hd_{p,i}$ and $T_i$ for any $i \in C_i$}
\State{Let $Q = \bigcup_{j\in Q_i}\text{children}(j) $}
\State{Query oracle once for each point in $Q \cap\{j: T_j=0\}$}
\State{Set $T_i \leftarrow 1$, update $\hd_{p,i}$ for each $j \in Q \cap\{j: T_j=0\}$}
\State{$\backslash\backslash$ compute the set $\{j \in Q: d(p,j) \leq 2^i \}$}
\State{Known points: $K = \{j: \hd_{p,j}(T_j) + C_{\delta/n}(T_j) \leq 2^{i} \text{ or } \hd_{p,j}(T_j) - C_{\delta/n}(T_j) > 2^i\}$}
\While{$|K| \neq |Q|$}
    \State{Compute $j^\ast(t) = \arg\min_{j \not\in K} \hd_{p,j}(T_j) - C_{\delta/n}(T_j)$}
	\State{Call oracle $\query{p}{j^\ast}$ and update $T_{j^\ast}, \hd_{p, j^\ast}$, $K$}
% 	\State{Update $T_{i^\ast}, \hd_{p, i^\ast}$, $K$}
\EndWhile
\State{$\backslash\backslash$ If $d(p, Q) > 2^i$}
\State{$Q_{i-1} = \{j \in Q: \hd_{p,j}(T_j) + C_{\delta/n}(T_j) \leq 2^{i}\}$}
\If{$Q_{i-1} = \emptyset$ }
	\State{\textbf{Return:}``no parent found''}
\Else
    \State{lower = \texttt{Noisy-Insert}$(p, \T, Q_{i-1}, i-1, \delta, \query{\cdot}{\cdot}, \{j \in Q: \hd_{p,j}(T_j)\}, \{j \in Q: T_j\} )$}
	\If{$Q_i \cap Q_{i-1} \neq \emptyset$ and lower = ``no parent found''}
		\State{Choose any $p' \in Q_i \cap Q_{i-1}$}
% 		\State{$\backslash\backslash$ Assign a parent to $p$}
		\State{Insert $p$ in $\text{children}(p')$ $\backslash\backslash$ Assign a parent to $p$}
		\State{\textbf{Return:}``parent found''}
	\Else 
		\State{\textbf{Return:}``no parent found''}
	\EndIf
\EndIf
\end{algorithmic}
\end{algorithm}

We provide pseudocode in Algorithm~\ref{alg:nns_insert}. The algorithm draws on ideas from \cite{beygelzimer2006cover} but includes additional logic to handle uncertainty. In particular, lines 2-8 implement a simple thresholding bandit based on the techniques of \cite{jamieson2018bandit} for adaptive hypothesis testing with family-wise probability of error control. This allows us to identify all points within $2^i$ of the nearest in the set $Q$. 
This must be done for every candidate level $i$ that $p$ might be inserted into until it reaches level $i_p$. 

This requires careful handling to ensure the algorithm succeeds with high probability. Each time the thresholding operation is called, there is a chance that the algorithm makes a mistake. \texttt{Noisy-Insert} is recursive and performs a this operation in every recursive call. Hence, if it makes an error on any call, the algorithm may fail. Thus, we must ensure that the probability of \texttt{Noisy-Insert} making an error in any recursive call is at most $\delta$. Since the level $i_p$ where $p$ is added is unknown and depends on the query point $p$, the number of recursive calls before success is unknown to the algorithm. Therefore, it is not a priori obvious how to perform the appropriate Bonferroni correction to ensure the probability of an error in any recursive call is bounded by $\delta$. A seemingly attractive approach is to use a summable sequence of $\delta_i$ depending on the level $i$ such that $\sum_{i}\delta_i = \delta$. For instance, $\delta_i = \delta\cdot2^{-i}$ would be suitable if the root of $\T$ is at level $1$. However, due to repeated calls to lines 2-8, this would lead to an additional multiplicative factor of the height of the cover tree affecting the sample complexity. 

Instead, \texttt{Noisy-Insert} \textit{shares} samples between rounds. By the nesting invariance, we have that $C_i \subset C_{i-1}$. Therefore, when we descend the tree from level $i$ to $i-1$, we already have samples of the distance of some points in $C_{i-1}$ to $p$. We simply reuse these samples and share them from round to round. Furthermore, since $\T$ is assumed to be a cover tree on $n$ points, we trivially union bound each confidence width to hold with probability $1-\delta/n$ such that all bounds for all recursive calls holds with probability at least $1-\delta$.

\section{Theoretical Guarantees of \texttt{Bandit Cover Tree}}
We measure performance along several axes: 1) how much memory is necessary to store the data structure, 2) how many calls to the distance oracle are needed at query time, 3) How many calls to the distance oracle are needed to for construction, and 4) the accuracy of the data structure in returning the correct nearest neighbor and performing other operations. Since we only have access to a stochastic oracle, ensuring the accuracy of the \texttt{Bandit Cover Tree} (\texttt{BCT}) is especially delicate. 
% In this section, we provide theoretical guarantees for the performance of \texttt{BCT}s along all of these axes. 
% All proofs as well as theoretical guarantees for removing points is deferred to Appendix~\ref{sec:nns_supplementary}. 

% We wish to provide guarantees for the performance measures described in the introduction: query time, build time, insertion time, removal time, memory footprint, and accuracy. In this section, we show that \texttt{BCT}'s achieve state of the art performance on these metrics despite only having access to noisy data. All proofs are deferred to Section~\ref{sec:nns_supplementary}.
% As \texttt{BCT} is adapted from cover trees directly, it inherits many properties and theoretical guarantees proven in \cite{beygelzimer2006cover}. The core challenge becomes proving correctness for the additions we have made to make the algorithm robust to noise and accounting for the number of extra calls to the distance oracle by the algorithm. The algorithms were stated for clarity as having a root at level $\infty$ and descending to level $\infty$. All analysis will be conducted with respect to a tree with root at $i_\text{top}$ and $i_\text{bottom}$ however, as a real tree cannot be infinitely tall. 

\subsection{Memory}\label{sec:nns_memory}
We begin by showing that a cover tree can efficiently be stored. Naively, a cover tree $\T$ on $\X$ can be stored using $O(n(i_{\text{top}} - i_\text{bottom}))$ memory where $n = |\X|$ and $i_{\text{top}} - i_\text{bottom}$ is the height of the tree. This follows from each level having at most $n$ nodes trivially and there being $i_{\text{top}} - i_\text{bottom}$ levels. For a well balanced tree, we expect that $i_{\text{top}} - i_\text{bottom}  = O(\log(n))$ leading to an overall memory complexity of $O(n\log(n))$. In fact, it is possible to do better. 

\begin{lemma}\label{lem:nns_memory}
A bandit cover tree requires $O(n)$ space to be stored. 
\end{lemma}

$O(n)$ memory is possible due to the nesting and covering tree invariants. By the nesting invariant, if a point $p$ is present in the $i^\text{th}$ cover $C_i$, then it is present in the cover sets of all lower levels. By the covering tree invariant, each point has a unique parent in the tree. Therefore, to store a cover tree, one need only store 1) which level of the tree each point first appears in 2) each point's parent node in the level above where it appears. After a node first appears in the tree, it may be represented implicitly in all lower levels. In particular, when all methods query the children of any node $p$, we define $p\in \text{children}(p)$. Formally, we define an implicit and explicit node as follows:
\begin{definition}\label{def:implicit_explicity}
A node $p \in C_i$ in level $i$ of cover tree $\T$ is \emph{explicit} if $p \not\in \{C_{\infty}, \cdots, C_{i+1}\}$ (i.e. it is not present in all previous levels.) Otherwise, any $p \in C_i$ is referred to as \emph{implicit}.
\end{definition}

\subsection{Accuracy}
 Next, we show that \texttt{BCT} is accurate with high probability. This requires three guarantees: 
\begin{enumerate}
\item Search accuracy: Given a cover tree $\T$ on set $\X$ and query point $\bq$, \texttt{Noisy-Find-Nearest} correctly identifies $\bq$'s nearest neighbor with high probability.
\item Insertion accuracy: Given a cover tree $\T$ and a point $p$ to be inserted, \texttt{Noisy-Insert} returns a valid cover tree that includes $p$ with high probability. 
% This will ensure accuracy in constructing a cover tree from scratch since \texttt{Noisy-Insert} can be called repeatedly to build a cover tree on a set $\X$. 
\item Removal accuracy: Given a cover tree $\T$ and a point $p$ to be removed, \texttt{Noisy-Remove} returns a valid cover tree that without $p$ (deferred to appendix). 
\end{enumerate}

\subsubsection{Search Accuracy}
We begin by showing that for any $\bq \in (\M, d)$ \texttt{Noisy-Find-Nearest} returns $\bx_q$ with high probability. 
The proof is deferred to the appendix, but the argument is sketched as follows. First, by appealing to results from \cite{mason2020finding}, we can guarantee that \texttt{Identify-Cover} succeeds with probability $1-\delta/\alpha$. Second, we show that for the choice of $\alpha$ in \texttt{Noisy-Find-Nearest}, the probability that any call to \texttt{Identify-Cover} fails is bounded by $1-\delta$. Finally, we show that if we correctly identify the necessary cover sets, which happens with probability at least $1-\delta$, we my appeal to Theorem 2 of \cite{beygelzimer2006cover} to ensure that our method returns $\bx_q$. 

% \texttt{Noisy-Find-Nearest} is an extension of the \texttt{Find-Nearest} algorithm from \cite{beygelzimer2006cover}. It follows the same logic as \texttt{Find-Nearest} but with an additional routine-- \texttt{Identify-Cover}-- to compute cover sets with from noisy data. 
% By Theorem 2 of \cite{beygelzimer2006cover}, \texttt{Find-Nearest} succeeds with probability 1. If \texttt{Identify-Cover} correctly returns the cover set each time that it is called in \texttt{Noisy-Find-Nearest}, then \texttt{Noisy-Find-Nearest} will correctly return $\bx_q$. Hence, the following guarantee is derived by showing that the probability of that \texttt{Identify-Cover} makes an error is small. 

\begin{lemma}\label{lem:nns_search_acc}
	Fix any $\delta \leq 1/2$ and a query point $\bq$. Let $\T$ be a cover tree on a set $\X$. \texttt{Noisy-Find-Nearest} returns $\bx_q \in \X$ with probability at least $1-\delta$. 
\end{lemma}

\subsubsection{Insertion Accuracy}
The proof that \texttt{Noisy-Insert} succeeds with high probability in adding any point $p$ to a cover tree $\T$ follows the argument of Lemma~\ref{lem:nns_search_acc} similarly. In particular, given a set $Q$, it hinges on the idea that we can use a threshold bandit to identify the set $\{j \in Q: d(p, j) \leq 2^i\}$ correctly with high probability by calling the stochastic oracle.
% where $Q$ is an arbitrary subset of children nodes. 
If this procedure succeeds, then \texttt{Noisy-Insert} has correctly identified the subset of $Q$ that are within $2^i$ of the point $p$ to be inserted. If $\{j \in Q: d(p, j) \leq 2^i\} = \emptyset$, we may instead verify that $d(p, Q) > 2^i$ and can compute a new cover set $Q_{i-1}$. The following lemma ensures that \texttt{Noisy-Insert} succeeds with high probability. 

\begin{lemma}\label{lem:nns_insert_acc}
	Fix any $\delta> 0$. Let $\T$ be a cover tree on a set $\X$ and $p$ be a point to insert. \texttt{Noisy-Insert} correctly returns a cover tree on $\X \cup \{p\}$ with probability at least $1-\delta$. 
\end{lemma}

To construct a cover tree from scratch given a set $\X$, one need only call \texttt{Noisy-Insert} $n$ times, once for each point in $\X$ beginning with an empty tree and run each call with failure probability $\delta/n$. The following corollary ensures that this leads to a cover tree satisfying the three invariances with probability $1-\delta$. 

\begin{corollary}[Construction of a Cover Tree]
Fix any $\delta > 0$ and a set $\X \in (\M, d)$. Calling \texttt{Noisy-Insert} with failure probability $\delta/n$ iteratively over the points in $\X$ yields a cover tree that satisfies the nesting, covering tree, and separation invariances with probability $1-\delta$. 
\end{corollary}

% The following lemma provides a similar guarantee for the \texttt{Noisy-Remove} method provided in the supplementary. The argument is similar to that of Lemma~\ref{lem:nns_insert_acc} in that it relies on correctly identifying the set $\{i \in Q: d(p, i) \leq 2^i\}$. 

% \begin{lemma}\label{lem:nns_remove_acc}
% 	Fix any $\delta> 0$. Let $\T$ be a cover tree on a set $\X$ and $p\in \X$ be a point to remove. \texttt{Noisy-Remove} correctly returns a cover tree on $\X \backslash \{p\}$ with probability at least $1-\delta$.
% \end{lemma}

\subsection{Query Time Complexity}

In the previous section, we proved that the algorithm succeeds with probability $1-\delta$ for search and insertion, with removal deferred to the appendix. In this section we begin to answer the question of how many calls to the stochastic distance oracle it requires to perform these operations.  
% \rob{\bf why would it be helpful to know $c$?}

% As an example, for a set of points in arranged uniformly on a surface in $\R^d$ have an expansion constant of at most $2^d$ \cite{beygelzimer2006cover}. 

We begin by analyzing the query time complexity of \texttt{Bandit Cover Tree}: the number of calls to the distance oracle made when answering a nearest neighbor query. To do so, we will make use of the \emph{expansion constant}, a data-dependent measure of dimensionality. In particular, for a query point $\bq$. We assume that the set $\X\cup \{\bq\}$ has an expansion constant of $c$ as defined in Definition~\ref{def:expansion_constant}. Note that this quantity is for analysis purposes only and is not required by the algorithm. \texttt{Noisy-Find-Nearest} can take in the expansion constant or a bound on it if it is known, but this is not required by the algorithm. In particular, knowing a bound on $c$ helps control the height of the tree. 
This allows \texttt{Noisy-Find-Nearest} to use a smaller union bound and leads to a tighter dependence on $n$. 

To bound query time, we appeal to the concept of explicit and implicit nodes as discussed in Section \ref{sec:nns_memory} and Definition~\ref{def:implicit_explicity}. Each point in $\X$ may correspond to multiple nodes in the tree. The first time a point appears as a node at the highest level where it is present, we say that it is \emph{explicitly represented} and is \emph{implicity represented} in all levels thereafter by the nesting invariant of cover trees. Recall that the set of nodes at each level $i$ of the cover tree is denoted $C_i$. \texttt{Noisy-Find-Nearest} proceeds by computing a cover set $Q_i \subset C_i$ at each level of the tree. Extending the concept of explicit and implicit representations of nodes, we say that a cover set $Q_i$ is implicitly represented if it only contains nodes that are implicitly represented. This plays an important role in our computation of query time. We may use the size of the last explicitly represented cover to help bound to complexity of the \texttt{Identify-Cover} subroutine in Algorithm~\ref{alg:nns_build_cover}. This routine is called for every level of the tree. Then, by bounding the height of the tree, we can control the total number of calls to the distance oracle. 

\begin{theorem}\label{thm:nns_query_time_bnd}
Fix $\delta < 1/2$, a cover tree $\T$ on set $\X$, and a query point $\bq$. Let $|\X| = n$ and assume that the expansion rate of $\X\cup \{\bq\}$ is $c$ (unknown to the algorithm). If \texttt{Noisy-Find-Nearest} succeeds, which occurs with probability $1-\delta$, then \texttt{Noisy-Find-Nearest} returns $\bq$'s nearest neighbor in at most 
\begin{align*}
	O\left( c^{17}\log(n)\log\left(\frac{n}{\delta}\right) {\kappa} \right)
\end{align*}
calls to the stochastic distance oracle
where parameter ${\kappa}$ is defined in the proof and satisfies
$$\kappa \ \leq\  B\, \max\big\{[d(\bq, \bx_q) - d(\bq, \bx_{q'})]^{-2} \, , \, d_{\min}^{-2}\big\}$$
where $\bx_{q'}$ is $\bq$'s second nearest neighbor, $d_{\min} := \min_{j,k}(\bx_j, \bx_k)$ is the smallest pairwise distance in $\X$, and $B>0$ is a universal constant.
\end{theorem}

\begin{remark}
The dependence of $c^{17}$ can be improved to $c^7$ matching \cite{beygelzimer2006cover} by redefining $L_t = \min_j \hd_{q,j}(T_j) - C_{\delta/|Q|}(T_j) + 2^i$ in the pseudocode of \texttt{Identify-Cover}. For most instances, this will lead to better performance, though a potentially worse worst-case bound on $\kappa$ in pathological cases. 
\end{remark}

% \begin{remark}
% The term $\overline{\kappa}$ captures the average effect of noise on this problem. It is similar to the term $\overline{\Delta^{-2}}$ in given in \cite{mason2019learning} for identifying the Nearest Neighbor Graph of $\X$ from a stochastic distance oracle.  
% As the noise variance goes to $0$, this term becomes $1$. 
% \end{remark}

% \begin{remark}
% One could use an assumption similar to Proposition 1 of \cite{bagaria2017medoids} to show that $\overline{\kappa} = O(\log(nd/\delta))$ by taking a distribution over the pairwise distances in $\X$. Unfortunately, this introduces a complex dependence between $n$, $d$, and $c$. 
% \end{remark}
The above result scales as $O\left(c^{O(1)}\log^2(n)\right)$. The following corollary highlights that if $c$ is known, this can be improved to $O\left(c^{O(1)}\log(n)\log(\log(n))\right)$.

% \rob{\bf I don't quite get the significance of the following corollary.  It seems worse than the result above.}
\begin{corollary}\label{cor:nns_query_time_bnd}
Under the same conditions as Theorem~\ref{thm:nns_query_time_bnd}, if any bound $\Tilde{c} \geq c$ on the expansion rate $c$ is given to the algorithm, the number of queries to the distance oracle is at most
\begin{align*}
	O\left( \Tilde{c}^{17}\log(n)\log\left(\frac{\Tilde{c}\log(n)}{\delta}\right) {\kappa} \right)
\end{align*}
calls to the stochastic oracle.
\end{corollary}

The term ${\kappa}$ captures the average effect of noise on this problem. It is similar to the term $\overline{\Delta^{-2}}$ in given in \cite{mason2019learning} for identifying the Nearest Neighbor Graph of $\X$ from a stochastic distance oracle. 
As the noise variance goes to $0$, this term becomes $1$. Furthermore, $\kappa$ is adaptive to the geometry of $\X$ and in most cases is much smaller than its bound in Theorem~\ref{thm:nns_query_time_bnd}. Intuitively, the bound on $\kappa$ states that the number of repeated samples is never worse that what is necessary to 1) distinguish between $\bq$'s nearest neighbor and second nearest neighbor and 2) distinguish between any two points in $\X$. Crucially, however, \texttt{Noisy-Find-Nearest} \emph{adapts} to the level of noise, the geometry of $\X$ and need not know these values a priori.

\subsection{Insertion Time and Build Time Complexity}
Next, we bound the number of calls to the distance oracle necessary to insert a new point $p$ into a cover tree $\T$. 
% We analyze the case that $p$ must be inserted within the tree. If $p$ instead is a new root of the tree, the same bound applies, though it is possible to prove a tighter bound for this special case. 

\begin{theorem}\label{thm:nns_insert_complex}
Fix $\delta > 0$, a cover tree $\T$ on set $\X$, and a point to insert $p$. Let $|\X| = n$ and assume that the expansion rate of $\X\cup \{p\}$ is $c$. Run \texttt{Noisy-Insert} with failure probability $1-\delta$ and pass it the root level cover set: $C_{i_\text{top}}$ and level $i = i_\text{top}$.
If \texttt{Noisy-Insert} succeeds, which occurs with probability $1-\delta$, then it returns a cover tree on $\X \cup \{p\}$ in at most
\begin{align*}
	O\left( c^7\log(n)\log\left(\frac{n}{\delta}\right) \overline{\kappa_p} \right)
\end{align*}
calls to the noisy distance oracle
where parameter $\overline{\kappa_{p}}$ is defined in the proof and depends on $\X$ and $p$. 
\end{theorem}

\begin{remark}
As in the statement of Theorem~\ref{thm:nns_query_time_bnd}, the term $\overline{\kappa_p}$ captures the average effect of noise on this problem, and as the noise variance goes to $0$, this term becomes $1$. 
\end{remark}

As discussed in Section~\ref{sec:nns_insert}, to construct a cover tree from scratch, one need only call \texttt{Noisy-Insert} on each point in $\X$ and add them the tree one at a time. The following theorem bounds the complexity of this process. 

\begin{theorem}\label{thm:nns_full_build_time}
Fix $\delta > 0$ and set $n$ points $\X$. Assume that the expansion rate of $\X$ is $c$. Calling \texttt{Noisy-Insert} with failure probability $\delta/n$ on each point in $\X$ one at a time, returns a cover tree $\T$ on $\X$ correctly with probability at least $1-\delta$ in at most
\begin{align*}
	O\left( c^7n\log(n)\log\left(\frac{n}{\delta}\right) \widetilde{\kappa} \right)
\end{align*}
calls to the noisy distance oracle
where $\widetilde{\kappa} := \frac{1}{n}\sum_{i\in \X}\overline{\kappa}_i$ for $\overline{\kappa}_i$ defined in the proof of Theorem~\ref{thm:nns_insert_complex}.
\end{theorem}

% \begin{remark}
% Note that the value of $\widetilde{\kappa}$ depends on the order in which points are inserted into the tree $\T$. It is possible that some insertion orders are more efficient than others. However, knowing the optimal order to insert points would require knowledge of the exact distances themselves, which is not available in this problem. Absent this knowledge, we assume that points are inserted in lexicographic order, with $\bx_1$ first and $\bx_n$ last. 
% \end{remark}

\subsection{Extension to the Nearest Neighbor Graph Problem}

In this section we extend the results of \cite{mason2019learning} for learning \emph{nearest neighbor graphs} from a stochastic distance oracle. In this setting, given a set $\X \in (\M, d)$, one wishes to learn the directed graph $G(\X)$ that connects each $\bx \in \X$ to its nearest neighbor in $\X\backslash\{\bx\}$. This problem is well studied in the noiseless regime and at times referred to as the all nearest neighbor problem \cite{clarkson1983fast, sankaranarayanan2007fast, vaidya1989ano}. \cite{mason2019learning} provide the first algorithm that learns the nearest neighbor graph using only a stochastic oracle and show that in special cases it achieves a rate of $O(n\log(n)\Delta^{-2})$ which is matches the optimal rate for the noiseless problem with only a multiplicative penalty of $\Delta^{-2}$ accounting for the effect of noise. Unfortunately, the condition necessary to show that result is stringent-- motivating the question of if similar performance can be achieved under more general conditions. We answer this question in the affirmative and show how to extend the \texttt{Bandit-Cover-Tree} to achieve near optimal performance on the nearest neighbor graph problem without the need for any additional assumptions on the data. This proceeds in two steps:
\begin{enumerate}
    \item Build a cover tree $\T$ on $\X$ with probability $1-\delta/2$. 
    \item For each $\bx \in \X$, find its nearest neighbor in $\X\backslash\{\bx\}$ using $\T$ with probability $1-\delta/2n$.
\end{enumerate}
The above is sufficient to specify each edge of the nearest neighbor graph. Note that $\T$ is a cover tree on $\X$ not $\X\backslash\{\bx\}$ so we cannot blindly use \texttt{Noisy-Find-Nearest} as it will return $\bx$ itself. To find $\bx$'s nearest neighbor in $\X\backslash\{\bx\}$, one may instead
modify \texttt{Noisy-Find-Nearest} so that \texttt{Identify-Cover} is called on the set $Q_{i}\backslash\{\bx\}$ instead of $Q_i$. Then, when the algorithm terminates, the final set $Q_{i_\text{bottom}}$ will have 2 points: $\bx$ and its nearest neighbor in $\X\backslash\{\bx\}$. A simple union bound ensures that this process succeeds with probability $1-\delta$. The following Lemma bounds the total number of samples. 
\begin{lemma}\label{lem:nns_nngraph_comp}
    Via the above procedure, \texttt{Bandit-Cover-Tree} returns the nearest neighbor graph of $\X$ with probability $1-\delta$ in 
    $$O\left(c^7n\log(n)\log\left(\frac{n}{\delta}\right) \widetilde{\kappa} + \sum_{\bx\in \X}c^7\log(n)\log\left(\frac{n}{\delta}\right) {\kappa}_{\bx}\right) $$
    for $\widetilde{\kappa}$ defined in Theorem~\ref{thm:nns_full_build_time} and ${\kappa}_x$ as in Theorem~\ref{thm:nns_query_time_bnd}. 
\end{lemma}
The proof follows by combining the guarantees of Theorems \ref{thm:nns_query_time_bnd} and \ref{thm:nns_full_build_time}. 
In particular, the above bound scales as $O(c^7n\log^2(n)\kappa)$ which matches the rate of \cite{mason2019learning} with an additional $\log$ factor. Importantly, the above result makes \emph{no assumptions} on the set $\X$, and instead the bound scales with the expansion rate, $c$. Hence, we achieve near optimal performance for the nearest neighbor graph problem with under \emph{far more general} conditions. 

\section{Conclusion}

In this paper, we introduced the \texttt{Bandit Cover Tree}  framework for the Nearest Neighbor Search under Uncertainty (NNSU) problem. \texttt{BCT} builds on top of the \texttt{Cover Tree} algorithm by \cite{beygelzimer2006cover}. 
We present three methods, \texttt{Noisy-Find-Nearest}, \texttt{Noisy-Insert}, and \texttt{Noisy-Remove}, and bound their accuracy, memory footprint, build complexity, insertion and removal complexities, and query complexities. In particular, we show a query complexity that is $O(\log(n))$, insertion complexity that is $O(\log^2(n))$, removal complexity that is $O(\log^2(n))$, and a construction complexity of $O(n\log^2(n))$. The query complexity matches the state of the art $n$ dependence for the NNS problem of $O(\log(n))$ despite the added uncertainty of the NNSU problem. The additional $\log(n)$ term present in the insertion and removal guarantee stems from a union bound which is necessary when dealing with a stochastic oracle, though it may be possible that by giving the algorithm access to the expansion constant, this can be improved to an additional doubly logarithmic term instead. Hence, the insertion, construction, and removal complexities are also near the state of the art for NNS. Lastly a memory footprint of $O(n)$ and accuracy of $1-\delta$ are both optimal in this problem. In particular, a tree with $n$ leaves requires $\Omega(n)$ memory to store, and though a dependence on $\delta$ is unavoidable when dealing with a stochastic oracle, \texttt{BCT}s enjoy a probability $1-\delta$ probability of success for any specified $\delta > 0$. 

Note that we focus on controlling the number of calls to the distance oracle in this work and assume that an individual call requires $O(1)$ work. 
% This is subtly different from the guarantees of other works studying nearest neighbor search which control computational complexity directly. 
We expect that for practical problems the bulk of the computational effort will be many repeated calls to the oracle.
% Hence, controlling the number of calls to the oracle is a strong surrogate for the overall computational complexity. 
It is an open question for future work to control the computational complexity in the noisy regime considering all operations, not just calls to the stochastic oracle. 

Furthermore, the results depend strongly on the expansion rate, $c$. Some works such as \cite{haghiri2017comparison, dasgupta2013randomized} trade accuracy for improved dependence on $c$ or other measures of dimension. Instead, these algorithms guarantee that $\bx_q$ is correctly returned with probability at least $1 - \delta_{c, n}$ for any $\bq$.
Though $\delta_{c,n}$ is tunable, it often depends on $c$ or other parameters potentially unknown to the practitioner. These methods often achieve good empirical performance, however. It may be interesting to develop methods that achieve lower theoretical accuracy but enjoy a softer dependence on $c$. 
% One especially promising approach in this vein is to modify the \texttt{Spill Tree} algorithm from \cite{dasgupta2013randomized}. A generic method to do \topk identification in multi-armed bandits such as \texttt{LUCB} from \cite{kalyanakrishnan2012pac} could be used as a subroutine to handle all calls to the distance oracle. 

% Lastly, we note that \texttt{Bandit Cover Tree} can be used to solve the nearest neighbor graph problem presented in \cite{mason2019learning}. In particular, given a set of points $\X$, first build a cover tree $\T$ on $\X$. This can be done in $O(n\log(n))$ calls to the distance oracle. Next, modify \texttt{Noisy-Find-Nearest} to find 2-nearest neighbors as discussed in \cite{beygelzimer2006cover}. For each $p \in \X$, $p$ will be its own nearest neighbor in $\T$ since $p$ is contained in the graph, but $p^\ast = \arg\min_{j \in \X\backslash\{p\}}d_{p,j}$ will be its second nearest neighbor. From this, we may connect $p$ to $p^\ast$ in the nearest neighbor graph. This process can be repeated for each $p\in \X$.
% Each 2-nearest neighbor query will require $O(\log(n))$ calls to the distance oracle. Summing this over the $n$ points in $\X$ with the $O(n\log(n))$ build complexity gives an overall complexity of $O(n\log(n))$ to learn a nearest neighbor graph. This matches the optimal rate given in Theorem~\ref{lem:nlogn} bounding the complexity of \texttt{ANNTri} but under \emph{far more general} conditions on the set $\X$. 

\bibliography{refs}

\clearpage
\onecolumn
\tableofcontents
\newpage
\appendix
\section{Additional Methods and Psuedocode}\label{sec:nns_supplementary}

\subsection{Finding the closest point in a set}

In this subsection, we provide the additional method used by \texttt{Noisy-Find-Nearest} to find the closest point among a set of alternatives. 

\begin{algorithm}[tbh]\label{alg:find_smallest}
   \caption{\texttt\texttt{find-smallest-in-set}}
\begin{algorithmic}[1]
\Require{Set $Q$, query point $\bq$, callable distance oracle $\query{\cdot}{\cdot}$}
\If{$|Q| = 1$}
\Return{$i\in Q$}
\EndIf
\State{Call oracle for each $j \in Q$ and initialize $T_j \rightarrow 1$}
\State{Define $j^\ast = \arg\min_{j \in Q}\hd_{q,j}(T_{q,j}) + C_{\delta/n}(T_{q,j})$}
\State{Define set $K = \{j: \hd_{q,j^\ast}(T_{q,j^\ast}) + C_{\delta/n}(T_{q,j^\ast}) < \hd_{q,j}(T_{q,j}) - C_{\delta/n}(T_{q,j})\}$}
	\While{$|Q \backslash K| > 1$}
	    \State{Call oracle $\query{q}{j}$ for every $j \in Q \backslash K$, update $T_j$s and set $K$}
	\EndWhile
	\Return{Return $i \in Q\backslash K$}
\end{algorithmic}
\end{algorithm}

\subsection{Approximate Nearest Neighbors}

Here we provide psuedocode for the \texttt{Approx-Noisy-Find-Nearest} routine. It is similar to \texttt{Noisy-Find-Nearest} except that it uses a thresholding bandit similar to the one used in \texttt{Noisy-Insert} and \texttt{Noisy-Remove} to check the condition $d(\bq, Q_i) \geq 2^{i-1}(1+1/\epsilon)$. 
% Additionally, it uses a second subroutine, \texttt{find-smallest-in-set} to query the oracle to find the closest point in the set $Q_i$ if the condition holds. 

\begin{algorithm}[tbh]
   \caption{\texttt{Approx-Noisy-Find-Nearest} \label{alg:nns_approx_search} }
\begin{algorithmic}[1]
\Require{Cover tree $\T$, failure probability $\delta$, expansion constant $c$ if known, query point $\bq$, callable distance oracle $\query{\cdot}{\cdot}$, subroutines \texttt{Identify-Cover} and \texttt{find-smallest-in-set}.}
\State{Let $Q_\infty = C_\infty$, $\alpha = n$}
\For{$i=\infty $ down to $i = -\infty$}
	\State{Let $Q =\bigcup_{p \in Q_{i}} \text{children}(p) $ }
	\If{$c$ is known:}
		\State{Let $\alpha =\min\left\{\left\lceil\frac{\log(n)}{\log(1 + 1/c^2)} + 1\right\rceil, n\right\}$}
	\EndIf
	\State{$\backslash\backslash$ Identify the set: $\{j\in Q: d(q, j) \leq \min_{k\in Q}d(q,k) + 2^{i-1}\}$}
	\State{$Q_{i-1} = \texttt{Identify-Cover}\left(Q, \delta/\alpha, \query{\cdot}{\cdot}, \bq, i\right)$}
	\State{$\backslash\backslash$ Check if $d(q, Q_{i-1}) \geq 2^i(1+1/\epsilon)$}
	\State{Call oracle for each $j \in Q_{i-1}$ and initialize $T_j \rightarrow 1$}
	\State{Define set $K = \{j: \hd_{q,j}(T_{q,j}) + C_{\delta/n}(T_{q,j}) \leq 2^{i}(1+1/\epsilon) \text{ or } \hd_{q,j}(T_{q,j}) - C_{\delta/n}(T_{q,j}) > 2^{i}(1+1/\epsilon)\}$}
	\While{$|K| \neq |Q_{i-1}|$}
	    \State{Call oracle $\query{q}{j^\ast}$ for  $j^\ast(t) = \arg\min_{j\not\in K} \hd_{q,j}(T_{q,j}) - C_{\delta/n}(T_{q,j})$}
	    \If{$\exists j \in Q_{i-1}:\hd_{q,j}(T_{q,j}) + C_{\delta/n}(T_{q,j}) \leq 2^i(1+1/\epsilon)$}
	    \State{\textbf{Continue} to next iteration of For loop}
	    \EndIf
	\EndWhile
	\State{$\backslash\backslash$ We have satisfied the condition that $d(\bq, Q_{i-1}) \geq 2^i(1+1/\epsilon)$ and can return the closest.}
	
	\Return{\texttt{find-smallest-in-set}$(Q_{i-1})$}
\EndFor \\
\Return{$Q_{-\infty}$, a singleton set containing $\bx_q$.}
\end{algorithmic}
\end{algorithm}

% \begin{algorithm}[tbh]\label{alg:find_smallest}
%   \caption{\texttt\texttt{find-smallest-in-set}}
% \begin{algorithmic}[1]
% \Require{Set $Q$, query point $\bq$, callable distance oracle $\query{\cdot}{\cdot}$}
% \If{$|Q| = 1$}
% \Return{$i\in Q$}
% \EndIf
% \State{Call oracle for each $j \in Q$ and initialize $T_j \rightarrow 1$}
% \State{Define $j^\ast = \arg\min_{j \in Q}\hd_{q,j}(T_{q,j}) + C_{\delta/n}(T_{q,j})$}
% \State{Define set $K = \{j: \hd_{q,j^\ast}(T_{q,j^\ast}) + C_{\delta/n}(T_{q,j^\ast}) < \hd_{q,j}(T_{q,j}) - C_{\delta/n}(T_{q,j})\}$}
% 	\While{$|Q \backslash K| > 1$}
% 	    \State{Call oracle $\query{q}{j}$ for every $j \in Q \backslash K$, update $T_j$s and set $K$}
% 	\EndWhile
% 	\Return{Return $i \in Q\backslash K$}
% \end{algorithmic}
% \end{algorithm}

\subsection{Algorithm to remove a point from a cover tree}

\subsubsection{Algorithm}

Given a cover tree $\T$ and a point $p$ that exists in the tree, we show how to remove $p$ efficiently such that the resulting tree also forms a cover tree the obeys the nesting, covering tree, and separation invariances. The process of removal is similar to insertion but slightly more complicated as we must find new parents for all of $p$'s children in $\T$. We provide pseudocode in Algorithm~\ref{alg:nns_remove}. Similar to \texttt{Noisy-Insert} it adapts a threshold bandit routine to minimize the number of calls to the stochastic oracle while controlling the probability of error.

\begin{algorithm}
   \caption{\texttt{Noisy-Remove} \label{alg:nns_remove} }
\begin{algorithmic}[1]
\Require{Cover tree $\T$ on $n$ points, past cover sets $\{Q_i, \cdots, Q_\infty\}$, failure probability $\delta$, point $p$ to be removed, callable distance oracle $\query{\cdot}{\cdot}$, level $i$}
\Require{Empirical estimates $\hd_{p,j}$ and $T_j$ for all $j \in C_i$ $\backslash\backslash$ let both be $0$ if no samples have been collected}
\State{Let $Q = \bigcup_{p\in Q_i}\text{children}(p) $}
%\State{Query oracle once for each point in $Q \cap\{i: T_i=0\}$}
%\State{Initialize $T_i \leftarrow 1$, update $\hd_{p,i}$ for each $i \in Q \cap\{i: T_i=0\}$}
\State{$\backslash\backslash$ Remove from lower levels first}
\State{$\T,  \{j,k: \hd_{j,k}(T_{j,k})\}, \{j,k: T_{j,k}\} \leftarrow $  \texttt{Noisy-Remove}$(\T, \{Q_{i-1}, \cdots, Q_\infty\}, \delta, p, \query{\cdot}{\cdot}, i-1, \{j,k: \hd_{j,k}(T_{j,k})\}, \{j,k : T_{j,k}\} )$}
\If{$p \in Q$}
	\State{Remove $p$ from $C_{i-1}$ and $\text{children}(\text{parent}(p))$  }
	\For{$q \in \text{children}(p)$}
		\State{$i' \leftarrow i-1$}	
		\State{$\backslash\backslash$ compute the set $\{j \in Q_{i'}: d(q,j) \leq 2^{i'} \}$}
		\State{Query oracle once for each point in $Q \cap\{j: T_{q,j}=0\}$}
		\State{Initialize $T_{q,j} \leftarrow 1$, update $\hd_{q,j}$ for each $j \in Q \cap\{j: T_{q,j}=0\}$}
		\State{Known points: $K = \{j: \hd_{q,j}(T_{q,j}) + C_{\delta/n^2}(T_{q,j}) \leq 2^{i'} \text{ or } \hd_{q,j}(T_{q,j}) - C_{\delta/n^2}(T_{q,j}) > 2^{i'}\}$}
		\While{$|K| \neq |Q_{i'}|$}
			\State{Call oracle $\query{q}{j^\ast}$ for  $j^\ast(t) = \arg\min_{j\not\in K} \hd_{q,j}(T_{q,j}) - C_{\delta/n^2}(T_{q,j})$}
			\State{Update $T_{q,j^\ast}, \hd_{q, j^\ast}$}
			\State{Update set $K$}
		\EndWhile
		\While{$ \{j \in Q_{i'}: \hd_{q,j}(T_{q,j}) + C_{\delta/n^2}(T_{q,j}) \leq 2^{i'} \} = \emptyset  $}
			\State{Add $q$ to the sets $C_{i'}$ and $Q_{i'}$. Increment $i'$. }
			\State{$\backslash\backslash$ for the incremented $i'$, recompute the set $\{j \in Q_{i'}: d(q,j) \leq 2^{i'} \}$}
			\State{Known points: $K = \{j: \hd_{q,j}(T_{q,j}) + C_{\delta/n^2}(T_{q,j}) \leq 2^{i'} \text{ or } \hd_{q,j}(T_{q,j}) - C_{\delta/n^2}(T_{q,j}) > 2^{i'}\}$}
			\While{$|K| \neq |Q_{i'}|$}
				\State{Call oracle $\query{q}{j^\ast}$ for  $j^\ast(t) = \arg\min_{j\not\in K} \hd_{q,j}(T_j) - C_{\delta/n^2}(T_j)$}
				\State{Update $T_{q,j^\ast}, \hd_{q, j^\ast}$}
				\State{Update set $K$}
			\EndWhile
		\EndWhile
		\State{Choose any $q' \in  \{j \in Q_{i'}: \hd_{q,j}(T_j) + C_{\delta/n}(T_j) \leq 2^{i'} \}$}
		\State{make $q' = \text{parent}(q)$}
	\EndFor
\EndIf
\end{algorithmic}
\end{algorithm}

Note that when we compute the distance to the set $Q_{i'}$ in lines 13 and 20, due to the nesting invariance of cover trees and the fact that samples are reused between rounds, $T_i > 0$ for all $i\in Q_{i'}$. Hence, it is unnecessary to collect any initial samples. Note that \texttt{Noisy-Remove} not only queries distances to the point $p$ to be removed but also to $p$'s children in order to find them new parents. Because of this, we index samples as $T_{i,j}$ denoting the number of calls to the distance oracle $\query{i}{j}$. Furthermore, we union bound with $\delta/n^2$ instead of $\delta/n$ where the additional factor of $n$ derives from a trivial bound that $p$ has at most $n-1$ children nodes. If the expansion rate is known, it is possible to union bound as $\delta/c^4n$ since Lemma~\ref{lem:num_children_bound} bounds the number of children as $c^4$ for any node $p$. However, as the other factor of $n$ remains unchanged by this, the improvement from $\log(n^2/\delta)$ to $\log(c^4n/\delta)$ only impacts constant factors in the sample complexity. 
% Hence, we ignore this. 
% It is technically possible to use a depth bound on $\T$ from Lemma 4.3 of \cite{beygelzimer2006cover} to achieve a dependence of $\log(c^6\log(n)/\delta)$. In particular, this reduces the $n$ dependence due to the union bound to be only doubly logarithmic. As the expansion constant is traditionally not known, we avoid the added complexity. 

\subsubsection{Theoretical Guarantees}

The following lemma guarantees that \texttt{Noisy-Remove} succeeds with high probability. The argument is similar to that of Lemma~\ref{lem:nns_insert_acc} in that it relies on correctly identifying the set $\{i \in Q: d(p, i) \leq 2^i\}$. 

\begin{lemma}\label{lem:nns_remove_acc}
	Fix any $\delta> 0$. Let $\T$ be a cover tree on a set $\X$ and $p\in \X$ be a point to remove. \texttt{Noisy-Remove} correctly returns a cover tree on $\X \backslash \{p\}$ with probability at least $1-\delta$.
\end{lemma}

Next we bound the number of calls to the distance oracle necessary to remove a point $p$ from a cover tree $\T$. 
% The proof of this, given in the appendix is markedly different from the corresponding proof of the noiseless case given in \cite{beygelzimer2006cover}. This is because in the noisy setting where we only have access to a stochastic oracle, we cannot compute distances exactly. 
% We analyze the case that $p$ must be inserted within the tree. If $p$ instead is a new root of the tree, the same bound applies, though it is possible to prove a tighter bound for this special case. 

\begin{theorem}\label{thm:nns_removal_complex}
Fix $\delta > 0$, a cover tree $\T$ on set $\X$, and a point $p \in \X$ to remove. Let $|\X| = n$ and assume that the expansion rate of $\X$ is $c$. Run \texttt{Noisy-Remove} with failure probability $1-\delta$ and pass it the root level cover set: $C_{i_\text{top}}$. If \texttt{Noisy-Remove} succeeds, which occurs with probability $1-\delta$, then it returns a cover tree on $\X \backslash \{p\}$ in at most
\begin{align*}
	O\left(c^{11}\log^2(n)\log\left(\frac{n^2}{\delta}\right)\widehat{\kappa}_p\right)
\end{align*}
calls to the noisy distance oracle
where parameter $\widehat{\kappa}_p$ is defined in the proof and depends on $\X$ and $p$. 
\end{theorem}
\begin{remark}
As in the statement of Theorem~\ref{thm:nns_query_time_bnd}, the term $\widehat{\kappa}_p$ captures the average effect of noise on this problem, and as the noise variance goes to $0$, this term becomes $1$. 
\end{remark}

\section{Proofs}
\subsection{Memory and accuracy proofs}

First, we prove Lemma~\ref{lem:nns_memory} which bounds the memory necessary to build, store, or search a \texttt{Bandit Cover Tree}

\begin{proof}[Proof of Lemma~\ref{lem:nns_memory}]
By the nesting invariance, once a point $p\in \X$ enters the cover tree at a level $i$, it is present in all lower levels. By the covering tree invariance, $p$ has a unique parent node. When computing estimates of distances, a single empirical estimate of $d(\bq, p)$ and confidence width is necessary to save for each point $p$. Hence, for each point $p \in \X$, at most $4$ numbers are necessary to store. Therefore, the \texttt{Bandit-Cover-Tree} can be stored in $O(n)$ memory.  

% By Theorem 1 of \cite{beygelzimer2006cover}, ordinary cover trees can be stored in $O(n)$ space. Hence, we must only show that any addition to any algorithm has not added more than $O(n)$ additional memory needed. In \texttt{Identify-Cover}, we trivially have that $|Q| \leq n$. In \texttt{Noisy-Insert} and \texttt{Noisy-Remove}, empirical estimates of distances and associated confidence widths are shared by all recursive calls to the algorithm. There are at most $n$ distance estimates and $n$ confidence widths, contributing $2n$ to the space requirement. By passing these values by reference rather than value for different recursive calls of each algorithm, the same $2n$ numbers may be shared for all recursive calls. The only other stored values, the indices for $i_\text{top}$ and $i_\text{bottom}$ contribute $O(1)$ space. Thus the total memory overhead for handling noisy estimates is $O(n)$ and the total memory requirement is $O(n)$ as well. 
\end{proof}

Next we prove several Lemmas that guarantee that the \texttt{Noisy-Find-Nearest}, \texttt{Noisy-Insert}, and \texttt{Noisy-Remove} methods all succeed with high probability. We begin with search accuracy. 

\begin{proof}[Proof of Lemma~\ref{lem:nns_search_acc}]
\texttt{Noisy-Find-Nearest} is adapted from \texttt{Find-Nearest}, but uses \texttt{Identify-Cover} to identify individual cover sets at it descends the tree and \texttt{find-smallest-in-set} to return the smallest in the final set.
By Theorem 2 of \cite{beygelzimer2006cover}, \texttt{Find-Nearest} succeeds with probability 1. Hence, we wish to show that the probability of \texttt{Identify-Cover} or \texttt{find-smallest-in-set} making an error in any call is at most $\delta$. 

% \texttt{Identify-Cover} is based on an inherits guarantees from the \st2 algorithm of \cite{mason2020finding} with the core difference the being that it is tuned to identify cover sets not $\epsilon$-good arms. 
By Theorem~\ref{thm:identify_cover_complex}, \texttt{Identify-Cover} given a failure probability of $\delta/\alpha$ succeeds with probability $1-\delta/\alpha$. Though \texttt{Identify-Cover} may return a set that is slightly larger than $\{j \in Q: d(\bq, \bx_j) \leq d(\bq, Q) + 2^i\}$, the returned set always contains this set if \texttt{Identify-Cover} succeeds. This implies that no grandparent node of $\bx_q$ is ever removed. Hence, this does not affect the proof of correctness. 

By definition, $\alpha = \min\left\{\left\lceil \frac{\log(n)}{\log(1 + 1/c^2)} + 1 \right\rceil, n\right\} + 1$. By Lemma~\ref{lem:height_bound}, $i_\text{top} - i_\text{bottom} \leq \alpha - 1$. Hence, \texttt{Identify-Cover} is called at most $\alpha - 1$ times (once per level during traversal). \texttt{find-smallest-in-set} is run a single time with failure probability $\delta/\alpha$, and by Theorem~\ref{thm:successive_elim_complex} makes an error with probability at most $\delta/\alpha$.
A union bound implies that the probability of an error in any call is at most $\delta$, completing the proof. 
\end{proof}

Next, we show that \texttt{Noisy-Insert} succeeds with high probability. 

\begin{proof}[Proof of Lemma~\ref{lem:nns_insert_acc}]
To show correctness, we begin by verifying that the thresholding bandit routine in lines $2-9$ does not fail in any recursive call of \texttt{Noisy-Insert}. Define the event 
$${\mathcal E} := \bigcap_{k\in [n]}\bigcap_{t=1}^\infty \left\{ |\hd_{p,k}(t)  - d_{p,k}| \leq C_{\delta/n}(t)\right\}$$
By definition of $C_{\delta}(t)$, we have that $P({\mathcal E}^c) \leq \delta$ where we have used the assumption that $|\X|=n$ so there are only $n$ different points explicitly represented in $\T$. For the remainder of the proof, we assume that ${\mathcal E}$ occurs and show that it leads to correctness in the thresholding bandit routine. 

Let $\mu > 0$ denote any threshold. On ${\mathcal E}$, we have that 
$$\left\{ k: d_{p,k} < \mu   \right\} \supset \left\{ k: \hd_{p,k} + C_{\delta/n}(T_k(t)) < \mu   \right\} $$
for any set of values $\{k: T_k(t)\}$ at any time $t$. Similarly, we have that 
$$\left\{ k: d_{p,k} > \mu   \right\} \supset \left\{ k: \hd_{p,k} - C_{\delta/n}(T_k(t)) > \mu   \right\}.$$
Therefore, if the algorithm stops sampling when either 
$$ \hd_{p,k} + C_{\delta/n}(T_k(t)) < \mu$$
or 
$$\hd_{p,k} - C_{\delta/n}(T_k(t)) > \mu $$
for every point $k$ and any value of $\mu$, then on event ${\mathcal E}$, we have identified the sets $\left\{ k: d_{p,k} < \mu   \right\}$ and $\left\{ k: d_{p,k} > \mu   \right\} $ correctly. In particular, the above superset relation holds with equality. 

Applying this to $\mu = 2^i$ for different values of $i$ in the algorithm we see that at all times the thresholding bandit implemented in lines $2-9$ succeeds. Since these bounds are shared between recursive calls of the algorithm, the routine succeeds in every recursive call.

We conclude by showing that if these sets have been computed correctly, which occurs when ${\mathcal E}$ occurs, then the algorithm correctly computes all quantities needed for the \texttt{Insert} algorithm in \cite{beygelzimer2006cover}.  

First, note that after the thresholding bandit terminates
$$\{\text{Threshold bandit has terminated} \} \cap \{j\in Q: \hd_{p,j}(T_j) + C_{\delta/n}(T_j) \leq 2^{i}\} = \emptyset \ \implies \ d(p,Q) > 2^i.$$

Next, note that when the thresholding bandit routine terminates at line $8$, we have that 
$$\left\{ k \in Q: \hd_{p,k} + C_{\delta/n}(T_k(t)) \leq 2^i  \right\} = \left\{ k \in Q: d_{p,k} \leq 2^i  \right\} = Q_{i-1}$$
where the last equality holds by definition in the \texttt{Insert} algorithm. 

Finally, by the nesting invariance of cover trees and the definition of children in a tree, we have that $Q_i \subset Q$. Hence $Q_i \cap Q_{i-1} \neq \emptyset$ is equivalent to the condition that $d(p, Q_i) \leq 2^i$. 
This implies that we have computed (under uncertainty) the same quantities needed in for the \texttt{Insert} algorithm of \cite{beygelzimer2006cover} for the NNS problem. Applying Theorem 3 therein completes the proof.
\end{proof}

Finally, we show that \texttt{Noisy-Remove} succeeds with high probability. 

\begin{proof}[Proof of Lemma~\ref{lem:nns_remove_acc}]
Similar to the proof for \texttt{Noisy-Insert}, we begin by showing correctness of the thresholding bandit in lines $9-16$. Again, we must verify that it does not fail in any recursive call of \texttt{Noisy-Remove}. Define the event 
$${\mathcal E} := \bigcap_{j,k\in [n]}\bigcap_{t=1}^\infty \left\{ |\hd_{j,k}(t)  - d_{j,k}| \leq C_{\delta/n^2}(t)\right\}$$
By definition of $C_{\delta}(t)$, we have that $P({\mathcal E}^c) \leq \delta$ where we have used the assumption that $|\X|=n$ so there are only $n$ different points explicitly represented in $\T$ and there are at most ${n \choose 2} $ pairs of distances. Note that especially for higher levels $i$ of the tree it is possible that $|Q_i| < n$. Hence a weaker union bound is possible. As $|Q_i|$ is unknown a priori, we take the naive bound that $|Q_i|\leq n$ for any round $i$, though it is possible to instead alter the union bound in different recursive calls to \texttt{Noisy-Remove}. 

For the remainder of the proof, we assume that ${\mathcal E}$ occurs and show that it leads to correctness in the thresholding bandit routine. 
Let $\mu > 0$ denote any threshold. On ${\mathcal E}$, we have that for any point $q$ (in particular any child of node $p$)
$$\left\{ k: d_{q,k} < \mu   \right\} \supset \left\{ k: \hd_{q,k}(T_{q,k}(t)) + C_{\delta/n^2}(T_{q,k}(t)) < \mu   \right\} $$
for any set of values $\{k: T_{q,k}(t)\}$ at any time $t$. Similarly, we have that 
$$\left\{ k: d_{p,k} > \mu   \right\} \supset \left\{ k: \hd_{q,k}(T_{q,k}(t)) - C_{\delta/n^2}(T_{q,k}(t)) > \mu   \right\}.$$
Therefore, if the algorithm stops sampling when either 
$$ \hd_{q,k}(T_{q,k}(t)) + C_{\delta/n^2}(T_{q,k}(t)) < \mu$$
or 
$$\hd_{q,k}(T_{q,k}(t)) - C_{\delta/n^2}(T_{q,k}(t)) > \mu $$
for every point $k$ and any value of $\mu$, then on event ${\mathcal E}$, we have identified the sets $\left\{ k: d_{p,k} < \mu   \right\}$ and $\left\{ k: d_{p,k} > \mu   \right\} $ correctly. In particular, the above superset relation holds with equality. 

Applying this to $\mu = 2^{i'}$ for different values of $i'$ in the algorithm we see that at all times, the thresholding bandit implemented in lines $9-16$ succeeds. Since these bounds are shared between recursive calls of the algorithm, the routine succeeds in every recursive call.

We conclude by showing that if these sets have been computed correctly, which occurs when ${\mathcal E}$ occurs, then the algorithm correctly computes all quantities needed for the \texttt{Remove} algorithm in \cite{beygelzimer2006cover}.  

First, note that when the thresholding bandit routine terminates
$$\{\text{Thresholding bandit terminated}\} \cap \{j\in Q: \hd_{q,j}(T_j) + C_{\delta/n^2}(T_{q,j}) \leq 2^{i'}\} = \emptyset \ \implies \ d(q,Q) > 2^{i'}.$$

Second, note that if the set 
$$ \{j \in Q_{i'}: \hd_{q,j}(T_{q,j}) + C_{\delta/n^2}(T_j) \leq 2^{i'} \}$$
is nonempty, then for any $q'$ contained in it, we have that $d_{q, q'} \leq 2^{i'}$. 
Therefore, on ${\mathcal E}$, \texttt{Noisy-Remove} computes the same quantities as \texttt{Remove}. 
Applying Theorem 4 of \cite{beygelzimer2006cover} completes the proof.
\end{proof}

\subsection{Query Time Complexity}

% \subsubsection{Proof of main theorem}

Now we turn our attention to the number of calls to the distance oracle needed by \texttt{Bandit Cover Tree}. We begin by proving a bound on the number of oracle calls made by \texttt{Noisy-Find-Nearest}. 

\begin{proof}[Proof of Theorem~\ref{thm:nns_query_time_bnd}]
Assume that \texttt{Noisy-Find-Nearest} succeeds, which occurs with probability $1-\delta$. Let $i_\text{top}$ and $i_\text{bottom}$ represent the top and bottom of $\T$. We begin by bounding the number of oracle calls drawn in an arbitrary round. Calls to the oracle only occur within the \texttt{Identify-Cover} routine. Suppose \texttt{Identify-Cover} is called on set $Q$ and run with failure probability $\delta/n$ since $\delta/\alpha \geq \delta/n$ deterministically. 
% \texttt{Identify-Cover} is based on an inherits guarantees from the \st2 \cite{mason2020finding} and the stochastic distance oracle obeys the assumptions of for stochastic oracles studied therein. 
By Theorem~\ref{thm:identify_cover_complex}
the number of calls to the distance oracle made by \texttt{Identify-Cover} is bounded by
\begin{align}\label{eq:identify_cover_complex}
c_1\log\left(\frac{n}{\delta}\right)\sum_{j\in Q} \min\left\{\max \left\{\frac{1}{(d_{\min} + 2^{i-1} - d_{q, j})^2}, \frac{1}{(d_{q,j} + \kappa_i - d_{\min})^2} \right\},  2^{-2i}\right\}
\end{align}
where $c_1$ includes constants and doubly logarithmic terms, $d_{\min} = \min_{j\in Q} d_{q,j}$, and $\kappa_i = \min |d_{\min} + 2^{i-1} - d_{q,j}|$. 
% $\kappa_i$ combines the $\alpha_\epsilon$ and $\beta_\epsilon$ terms in Theorem~\ref{thm:st2_complexity}. 

Define $\kappa_i^\text{avg}$ to be the arithmetic means of the summands in the above sum in Equation~\ref{eq:identify_cover_complex}. Recall that for the cover set $Q_i$ at level $i$, $Q$ is defined as $Q = \bigcup_{p \in Q_i}\text{children}(p)$ in the \texttt{Noisy-Find-Nearest} algorithm. Hence, we may compactly bound Equation~\ref{eq:identify_cover_complex} as 
$$c_1\log\left(\frac{n}{\delta}\right) \left|\bigcup_{p \in Q_i}\text{children}(p)\right| \kappa_i^\text{avg}$$
since we have replaces a $\min$ by one of the terms inside. 

Applying this, we may sum over all levels of the tree $\T$ and bound the total number of oracle calls as
\begin{align*}
c_1\log\left(\frac{n}{\delta}\right) \sum_{i = i_\text{top}}^{i_\text{bottom}} \left|\bigcup_{p \in Q_i}\text{children}(p)\right| \kappa_i^\text{avg}
\end{align*}
where we have written the outer sum index in order of \emph{descending} $i$ since $i_\text{top} > i_\text{bottom}$. This is done to reflect to the process of descending tree $\T$ and counting the number of oracle calls taken at each level. 
We proceed by bounding 
$ \left|\bigcup_{p \in Q_i}\text{children}(p)\right|.$

For the set $Q$ defined in the algorithm, define the set $$\Tilde{Q}_i : = \{j\in Q : d(\bq, \bx_j) \leq d(\bq, Q) + 2^{i}\}.$$
By design, \texttt{Identify-Cover} returns a set $Q_i$ such that 
$$\Tilde{Q_i} = \{j\in Q : d(\bq, \bx_j) \leq d(\bq, Q) + 2^{i}\} \subset Q_i \subset \{j\in Q : d(\bq, \bx_j) \leq d(\bq, Q) + 2^{i+1}\} = \Tilde{Q}_{i+1}. $$
By the expansion rate, $|\Tilde{Q}_{i+1}| \leq c|\Tilde{Q}_{i}|$ which implies that $|Q_{i}| \leq c|\Tilde{Q}_i|$. By Lemma~\ref{lem:num_children_bound}, the number of children of any node is bounded by $c^4$. Therefore, we have that 
\begin{align*}
    \left|\bigcup_{p \in Q_i}\text{children}(p)\right|& \leq \left|\bigcup_{p \in \Tilde{Q}_{i-1}}\text{children}(p)\right|\\ 
    &= \left|\bigcup_{p \in \Tilde{Q}_{i+1}\backslash \Tilde{Q}_i}\text{children}(p) \cup \bigcup_{p \in \Tilde{Q}_{i}}\text{children}(p) \right| \\
    & \stackrel{a}{\leq} c^4|\Tilde{Q}_{i+1}\backslash \Tilde{Q}_i|\left|\bigcup_{p \in \Tilde{Q}_{i}}\text{children}(p) \right|\\
    & \stackrel{b}{\leq}c^4(c-1)| \Tilde{Q}_i|\left|\bigcup_{p \in \Tilde{Q}_{i}}\text{children}(p) \right|\\
    & \stackrel{c}{\leq}c^5\left|\bigcup_{p \in \Tilde{Q}_{i}}\text{children}(p) \right|^2
\end{align*}
Inequality $a$ follows by the bound on the number of children of any node. Inequality $b$ is a consequence of the expansion rate on $\X\cup\{\bq\}$. To see this observe that we have that $|\Tilde{Q}_{i+1}| \leq c|\Tilde{Q}_i|$ and $|\Tilde{Q}_{i+1}| = |\Tilde{Q}_{i}| + |\Tilde{Q}_{i+1}\backslash\Tilde{Q}_i|$. Hence $|\Tilde{Q}_{i+1}\backslash\Tilde{Q}_i|\leq (c-1)|\Tilde{Q}_i|$. Inequality $c$ is a consequence of the nesting invariance which implies that $p \in \text{children}(p)$ for any node $p$ in the cover tree. 

Let $Q^\ast$ be the final explicit $Q$. That is $Q^\ast = \Tilde{Q}_{i^\ast}$ where $i^\ast$ is the lowest level such that an explicit node exists in the cover set $Q_{i^\ast}$. For all levels $i$ such that $i > i^\ast$ (levels above $i^\ast$), explicit nodes have yet to be added. Hence 
$$\left|\bigcup_{p \in \Tilde{Q}_i}\text{children}(p)\right| \leq \left|\bigcup_{p \in \Tilde{Q}_{i^\ast}}\text{children}(p)\right|. $$
For all levels below $i^\ast$, by definition, no new nodes are added as all remaining nodes are implicit. Therefore, for $i < i^\ast$ (lower levels of the tree), 
$$\left|\bigcup_{p \in \Tilde{Q}_i}\text{children}(p)\right| \leq \left|\bigcup_{p \in \Tilde{Q}_{i^\ast}}\text{children}(p)\right|. $$
In particular, $Q_{i^\ast}$ maximizes $\left|\bigcup_{p \in \Tilde{Q}_i}\text{children}(p)\right|$. By Lemma~\ref{lem:size_of_largest_set}, we have that $\left|\bigcup_{p \in \Tilde{Q}_{i^\ast}}\text{children}(p)\right| \leq c^5$. 
Next, for clarity, define $$\overline{\kappa} := \frac{1}{i_\text{top} - i_\text{bottom}}\sum_{i=i_\text{top}}^{i_\text{bottom}}\kappa_i^\text{avg},$$ the average of the $\kappa_i^\text{avg}$ terms that appears in each level. Plugging in both above pieces, we have that
\begin{align*}
c_1\log\left(\frac{n}{\delta}\right)  \sum_{i = i_\text{top}}^{i_\text{bottom}} \left|\bigcup_{p \in Q_i}\text{children}(p)\right| \kappa_i^\text{avg} 
\leq c_1c^5\log\left(\frac{n}{\delta}\right)  \sum_{i = i_\text{top}}^{i_\text{bottom}} \left|\bigcup_{p \in \Tilde{Q}_{i^\ast}}\text{children}(p)\right|^2 \kappa_i^\text{avg}
\leq c_1c^{15}\log\left(\frac{n}{\delta}\right) (i_\text{top} - i_\text{bottom})\overline{\kappa}
\end{align*}

By Lemma~\ref{lem:height_bound}, we have that $i_\text{top} - i_\text{bottom} = O(c^2\log(n))$. Plugging this in we have that the total number of oracle taken by \texttt{Identify-Cover} calls is bounded by 
\begin{align*}
	O\left( c^{17}\log(n)\log\left(\frac{n}{\delta}\right) \overline{\kappa} \right).
\end{align*}
It remains to bound the number or oracle calls drawn by \texttt{find-smallest-in-set} to find $\bx_q$ in the set $Q_{i_\text{bottom}}$. One the high probability event that all calls to \texttt{Identify-Cover} succeed, which occurs with probability at least $1-\delta$, Lemma~\ref{lem:nns_search_acc} implies that $\bx_q\in Q_{i_\text{bottom}}$. Hence, $d(\bq, Q_{i_\text{bottom}} = d(\bq, \bx_q)$. 
By Theorem~\ref{thm:successive_elim_complex}, the total number of samples drawn by \texttt{find-smallest-in-set} is bounded by 
\begin{align*}
    c_2\log\left(\frac{n}{\delta}\right)\sum_{j\in Q_{i_\text{bottom}}\backslash\{\bx_q\}}(d(\bq, \bx_q) - d(\bq, \bx_j))^{-2} = c_2\log\left(\frac{n}{\delta}\right)|Q_{i_\text{bottom}}|\kappa'\leq c_2\log\left(\frac{n}{\delta}\right)c^6\kappa'
\end{align*}
where $\kappa'$ is the empirical average of $(d(\bq, \bx_q) - d(\bq, \bx_j))^{-2}$ for all $j \in \backslash\{\bx_q\}$ and we have used the fact that $|Q_{i_\text{bottom}}|\leq c|\Tilde{Q}_{i^\ast}|\leq c^6$.
Therefore, the total complexity of \texttt{Noisy-Find-Nearest} is 
\begin{align*}
    O\left( c^{17}\log(n)\log\left(\frac{n}{\delta}\right) \overline{\kappa} + \log\left(\frac{n}{\delta}\right)c^6\kappa' \right). = O\left( c^{17}\log(n)\log\left(\frac{n}{\delta}\right) \kappa\right)
\end{align*}
where $\kappa = \max(\overline{\kappa}, \kappa')$. To complete the proof, note that $\kappa' \leq ((d(\bq, \bx_q) - d(\bq, \bx_{q'}))^{-2}$ where $\bx_{q'}$ is the second closest point to $\bq$. To bound $\overline{\kappa}$, note that
each summand in Equation~\eqref{eq:identify_cover_complex} may be bounded instead as $2^{-2i}$ in level $i$. As $\kappa_i^\text{avg}$ is an average of these terms, it too is bounded by $2^{-2i}$. $\overline{\kappa}$ is the average of $\kappa_i^\text{avg}$ in each round. As $2^{-2i}$ monotonically decreases in $i$, we have that $\overline{\kappa} \leq \max_{i}\kappa_i^\text{avg}\leq 2^{-2i_{\text{bottom}}}$.  
% To bound $i_\text{Bottom}$, note that cover 
% trees obey the same invariances as navigating nets from \cite{krauthgamer2004navigating}. 
Hence, by Lemma~\ref{lem:krauthgamer_height}, the bottom level is $i = O(\log_2(d_{\min}))$ where $d_{\min} :=\min_{j,k\in \X}d(\bx_j, \bx_k)$. 
Hence We have that $2^{-2i_{\text{bottom}}}$ is $O(d_{\min}^{-2})$. Therefore, $\kappa \leq O\left(\max\{[d(\bq, \bx_q) - d(\bq, \bx_{q'})]^{-2}, d_{\min}^{-2}\}\right)$. 
\end{proof}

\begin{proof}[Proof of Corollary~\ref{cor:nns_query_time_bnd}]
The proof follows identically, except
that we make use $\Tilde{c}$ as a bound on $c$ throughout and may plug in 
	$$\alpha = \left\lceil \frac{\log(n)}{\log(1+1/\Tilde{c}^2)} + 1 \right\rceil + 1 = O(\Tilde{c}^2\log(n)). $$
	Hence the term from the union bound becomes $O\left(\log\left(\frac{\Tilde{c}^2\log(n)}{\delta}\right)\right)$ instead of $\log\left(\frac{n}{\delta}\right)$. 
\end{proof}

\subsection{Insertion Time Complexity}

Next, we bound the number of oracle calls made by \texttt{Noisy-Insert}. 

\begin{proof}[Proof of Theorem~\ref{thm:nns_insert_complex}]
We begin by analyzing the complexity of the thresholding bandit subroutine in lines $2-9$ of \texttt{Noisy-Insert}. Assume the same event ${\mathcal E}$ from the proof of Lemma~\ref{lem:nns_insert_acc} that holds with probability $1-\delta$. We proceed by bounding the number of rounds any point $j$ in the set $Q$ may be $i^\ast(t)$ before it must enter the set $K$. Summing up the complexity for all points in $Q$ bounds the complexity of this routine. 

Suppose we wish to insert $p$ into level $i$. Assume for point $j \in C_i$ that $d_{p,j} \leq 2^i $. Assume that 
$$T_j \geq \frac{c'}{(d_{p,j} - 2^i)^2}\log\left( \frac{n}{\delta}\log\left(  \frac{n}{\delta(d_{p,j} - 2^i)^2}     \right)   \right)$$ 
for a sufficiently large constant $c'$. Then
\begin{align*}
\hd_{p,j}(T_j) + C_{\delta/n}(T_j) \stackrel{{\mathcal E}}{\leq} d_{p,j} + 2C_{\delta/n}(T_j) \stackrel{(a)}{<} d_{p,j} + 2^i - d_{p,j} = 2^i,
\end{align*}
implying that $j$ must be in the set $K$. Inequality $(a)$ follows from Lemma~\ref{lem:inv_conf_bounds}. This argument may be repeated for $j \in C_i$ such that $d_{p,j} > 2^i $ by instead considering the lower confidence bound on $\hd_{p,j}(T_j)$. 

Summing over all $j$ in the set $Q$, no more than
\begin{align*}
\sum_{j\in Q} \frac{c'}{(d_{p,j} - 2^i)^2}\log\left( \frac{n}{\delta}\log\left(  \frac{n}{\delta(d_{p,j} - 2^i)^2}     \right)   \right)
\leq c_1 |Q| \log\left(\frac{n}{\delta}\right)\kappa_i^\text{avg}
\end{align*}
calls to the oracle are made between lines 2 and 8 for a problem independent constant $c_1$.  Similar to the proof of Theorem~\ref{thm:nns_query_time_bnd}, we define $\kappa_i^\text{avg}$ to be average of the summands including doubly logarithmic terms for brevity and clarity. 

By definition, in level $i$ $Q = \bigcup_{j \in Q_i}\text{children}(j)$. In the worst case, $p$ is added at the leaf level, $i_\text{bottom}$. \texttt{Noisy-Insert} descends down the tree via recursive calls. This happens at most $i_\text{top} - t_\text{bottom}$ times. Summing over all levels, the total number of calls to the oracle is bounded by 
\begin{align*}
	c_1\log\left(\frac{n}{\delta}\right)\sum_{i = i_\text{top}}^{i_\text{bottom} } \left| \bigcup_{j \in Q_i}\text{children}(j) \right|  \kappa_i^\text{avg}
\end{align*}
where the sum is indexed from the largest $i$ to the smallest to reflect descending the tree. As in the proof of Theorem~\ref{thm:nns_query_time_bnd}, there exists a level $i^\ast$ which maximizes $\left| \bigcup_{j \in Q_i}\text{children}(j) \right|$ and we have that $\left| \bigcup_{j \in Q_{i^\ast}}\text{children}(j) \right| \leq c^5$. Define 
$$\overline{\kappa}_p = \frac{1}{i_\text{top} - i_\text{bottom}}\sum_{i=i_\text{top}}^{i_\text{bottom}}\kappa_i^\text{avg}$$
Plugging this in,
\begin{align*}
c_1\log\left(\frac{n}{\delta}\right)\sum_{i = i_\text{top}}^{i_\text{bottom} } \left| \bigcup_{j \in Q_i}\text{children}(j) \right| \kappa_i^\text{avg} & \leq c_1c^5 \log\left(\frac{n}{\delta}\right)(i_\text{top} - i_\text{bottom}) \overline{\kappa_p} \\ & \leq O\left(c^7\log(n)\log\left(\frac{n}{\delta}\right)\overline{\kappa_p} \right). 
\end{align*}
where the final inequality follows from Lemma~\ref{lem:height_bound}. 
\end{proof}

\begin{proof}[Proof of Theorem~\ref{thm:nns_full_build_time}]
Begin with an empty tree. \texttt{Noisy-Insert} can trivially place $\bx_1$ as a root and replace the root as necessary. We run \texttt{Noisy-Insert} with failure probability $\delta/n$. Since placing the first node at the root is trivial, make the inductive hypothesis that we have a correct tree $\T$ built on a strict subset $\S \subsetneq \X$ and wish to insert a point $p \in \X\backslash \S$ to $\T$ using \texttt{Noisy-Insert}. By Lemma~\ref{lem:nns_insert_acc}, this process succeeds with probability $1-\delta/n$. By Theorem~\ref{thm:nns_insert_complex}, this requires no more than 
\begin{align*}
	O\left( c^7\log(n)\log\left(\frac{n}{\delta}\right) \overline{\kappa_p} \right)
\end{align*}
calls to the noisy distance oracle. 

A union bound over inserting the $n$ points in $\X$ implies correctness. Summing the above expression for every point $p \in \X$ bounds the total sample complexity necessary for construction, stated in the Theorem. In particular, $\widetilde{\kappa}$ is the arithmetic mean of the individual $\overline{\kappa_p}$s. 
\end{proof}

\subsection{Removal Time Complexity}

Finally, we bound the number of oracle calls needed by \texttt{Noisy-Remove}

\begin{proof}[Proof of Theorem~\ref{thm:nns_removal_complex}]
We begin by analyzing the complexity of the Thresholding bandit subroutine in lines $9-16$ of \texttt{Noisy-Remove}. Assume the same event ${\mathcal E}$ from the proof of Lemma~\ref{lem:nns_remove_acc} that holds with probability $1-\delta$. Fix an arbitrary $q \in \text{children}(p)$. We proceed by bounding the number of rounds any point $j$ in the set $Q_{i'}$ may be $i^\ast(t)$ before it must enter the set $K$. Summing up the complexity for all points in $Q$ bounds the complexity of this routine. 

Assume for point $j \in C_{i'}$ that $d_{q,j} \leq 2^{i'} $. Assume that 
$$T_{q,j} \geq \frac{c'}{(d_{q,j} - 2^{i'})^2}\log\left( \frac{n^2}{\delta}\log\left(  \frac{n^2}{\delta(d_{q,j} - 2^{i'})^2}     \right)   \right)$$ 
for a sufficiently large constant $c'$. Then
\begin{align*}
\hd_{q,j}(T_{q,j}) + C_{\delta/n^2}(T_{q,j}) \stackrel{{\mathcal E}}{\leq} d_{q,j} + 2C_{\delta/n^2}(T_{q,j}) \stackrel{(a)}{<} d_{q,j} + 2^{i'} - d_{p,j} = 2^{i'},
\end{align*}
implying that $j$ must be in the set $K$. Inequality $(a)$ follows from Lemma~\ref{lem:inv_conf_bounds}. This argument may be repeated for $j \in C_i$ such that $d_{q,j} > 2^{i'} $ by instead considering the lower confidence bound on $\hd_{q,j}(T_{q,j})$. 

Summing over all $j$ in the set $Q_{i'}$, no more than
\begin{align*}
\sum_{j\in Q_{i'}} \frac{c'}{(d_{q,j} - 2^{i'})^2}\log\left( \frac{n^2}{\delta}\log\left(  \frac{n^2}{\delta(d_{q,j} - 2^{i'})^2}     \right)   \right)
\leq c_1 |Q_{i'}| \log\left(\frac{n^2}{\delta}\right)\kappa_{q,i'}^\text{avg}
\end{align*}
calls to the oracle are made between lines 9 and 16 for a problem independent constant $c_1$. Similar to the proof of Theorem~\ref{thm:nns_query_time_bnd}, we define $\kappa_{q,i'}^\text{avg}$ to be average of the summands including doubly logarithmic terms for brevity and clarity. 

If $$ \{j \in Q_{i'}: \hd_{q,j}(T_{q,j}) + C_{\delta/n^2}(T_{q,j}) \leq 2^{i'} \} = \emptyset,  $$
$i'$ is reset to $i'+1$ and the algorithm proceeds to the next level up the tree. An identical computation is performed as in lines $9-16$ and a similar bound applies as the above. The only difference is that since the Thresholding bandit is now comparing to a threshold of $2^{i' +1}$ (or alternatively $2^{i'}$ for the incremented value of $i'$) we incur a dependence on $\kappa_{q,i'+1}^\text{avg}$ instead. In particular, the number of oracle queries drawn between lines 20 and 25 is at most 
\begin{align*}
\sum_{j\in Q_{i'+1}} \frac{c'}{(d_{q,j} - 2^{i'+1})^2}\log\left( \frac{n^2}{\delta}\log\left(  \frac{n^2}{\delta(d_{q,j} - 2^{i'+1})^2}     \right)   \right)
\leq c_1 |Q_{i'+1}| \log\left(\frac{n^2}{\delta}\right)\kappa_{q,i'+1}^\text{avg}. 
\end{align*}
This process repeats until the conditional in the while loop is no longer satisfied. Naively, this happens at most $i_\text{top} - i + 1$ times as $i'$ is initialized as $i-1$ and is incremented until potentially it reaches the top level of $\T$, $i_\text{top}$. 
Summing this quantity over all levels of the tree, we may bound the total number of oracle calls as
\begin{align*}
\sum_{i' = i_\text{top}}^{i - 1} c_1|Q_{i'}| \log\left(\frac{n^2}{\delta}\right)\kappa_{q,i'}^\text{avg}. 
\end{align*}
As in the proof of Theorem~\ref{thm:nns_query_time_bnd}, there exists a level $i^\ast$ which maximizes $\left| \bigcup_{j \in Q_i}\text{children}(j) \right|$ and we have that $\left| \bigcup_{j \in Q_{i^\ast}}\text{children}(j) \right| \leq c^5$. As 
$Q_i \subset \bigcup_{j \in Q_i}\text{children}(j) $
by the nesting invariance, $c^5$ likewise bounds $Q_i$ for all $i$. 
Define 
$$\kappa_q^\text{avg}(i) = \frac{1}{i_\text{top} - i + 1}\sum_{i=i_\text{top}}^{i - 1}\kappa_{q,i}^\text{avg}$$
Plugging this in, we may bound the total number of oracle calls (for the distance of any point to $q$) by
\begin{align*}
\sum_{i' = i_{\text{top}}}^{i - 1} c_1|Q_{i'}| \log\left(\frac{n^2}{\delta}\right)\kappa_{q,i'}^\text{avg} & \leq c_1c^5 (i_\text{top} - i + 1)\log\left(\frac{n^2}{\delta}\right)\kappa_{q}^\text{avg}(i) \\
&\leq c_1 c^5 (i_\text{top} - i_\text{bottom})\log\left(\frac{n^2}{\delta}\right)\kappa_{q}^\text{avg}(i). 
\end{align*}
Next, by Lemma~\ref{lem:height_bound}, $i_\text{top} - i_\text{bottom}  = O(c^2\log(n))$. Therefore, 
\begin{align*}
c_1 c^5 (i_\text{top} - i_\text{bottom})\log\left(\frac{n^2}{\delta}\right)\kappa_{q}^\text{avg}(i) \leq O\left(c^7 \log(n)\log\left(\frac{n^2}{\delta}\right)\kappa_{q}^\text{avg}(i)\right)
\end{align*}
The above bounds the number of calls to the distance oracle needed for any child $q$ of $p$, the point to be removed. Due to the `For' loop in line 6, this process is repeated for all $q \in \text{children}(p)$. By Lemma~\ref{lem:num_children_bound}, the number of children of any node $p \in \T$ is at most $c^4$. 
Define 
$$\kappa^p(i) := \frac{1}{|\text{children}(p)|}\sum_{q \in \text{children}(p)} \kappa_q^\text{avg}(i)$$
where the superscript $p$ and the parenthetical $i$ denote that this quantity depends on all children of $p$ and level $i$ of the tree. 
Then 
\begin{align*}
\sum_{q \in \text{children}(p)} \kappa_q^\text{avg}(i) & =|\text{children}(p)| \frac{1}{|\text{children}(p)|}\sum_{q \in \text{children}(p)} \kappa_q^\text{avg}(i) \\
& \leq c^4 \frac{1}{|\text{children}(p)|}\sum_{q \in \text{children}(p)} \kappa_q^\text{avg}(i)\\
&= c^4 \kappa^p(i)
\end{align*}
Therefore, summing over all $q \in \text{children}(p)$, we can bound the total number of calls to the distance oracle drawn in lines 6 to 29 of \texttt{Noisy-Remove} by
\begin{align*}
O\left(c^9 \log(n)\log\left(\frac{n^2}{\delta}\right)\kappa^p(i)\right).
\end{align*}

As \texttt{Noisy-Remove} is recursive, it remains to sum the complexity of all recursive calls. The above bound depends on the level $i$ on which \texttt{Noisy-Remove} is called only through the term $\kappa^p(i)$. In the worst case, $p$ is present in every level and the `If' condition in line 4 is true for every recursive call. Hence, we sum the above expression over every level $i$. Define 
$$\widehat{\kappa}_p := \frac{1}{i_\text{top} - i_\text{bottom}}\sum_{i = i_\text{top}}^{i_\text{bottom}} \kappa^p(i).$$
The total number of oracle calls is bounded as
\begin{align*}
\sum_{i = i_\text{top}}^{\text{bottom}} O\left(c^9 \log(n)\log\left(\frac{n^2}{\delta}\right)\kappa^p(i)\right) 
& = O\left(c^9(i_\text{top} - i_\text{bottom}) \log(n)\log\left(\frac{n^2}{\delta}\right)\widehat{\kappa}_p\right) \\
& = O\left(c^{11}\log^2(n)\log\left(\frac{n^2}{\delta}\right)\widehat{\kappa}_p\right)
\end{align*}
completing the proof. 
\end{proof}

\section{External results used in proofs}

\subsection{Results about bandit subroutines}

\begin{theorem}[Based on Theorem 3.1 of \cite{mason2020finding}]\label{thm:identify_cover_complex} Fix $\delta < 1/2$, a query point $q$, a set of points $Q$, and $i\in \mathbb{Z}$. With probability at least $1-\delta$, \texttt{Identify-Cover} returns a set $Q_i$ such that 
$$\{j \in Q: d_{q,j} \leq d(\bq, Q) + 2^{i-1}\} \subset Q_i \subset \{j \in Q: d_{q,j} \leq d(\bq, Q) + 2^{i-1} + 2^{i-2}\}$$
in at most 
calls to the distance oracle made by \texttt{Identify-Cover} is bounded by
\begin{align}
c_1\log\left(\frac{n}{\delta}\right)\sum_{j\in Q} \min\left\{\max \left\{\frac{1}{(d(\bq, Q) + 2^{i-1} - d_{q, j})^2}, \frac{1}{(d_{q,j} + \kappa_i - d(\bq, Q))^2} \right\},  2^{-2i}\right\}
\end{align}
where $c_1$ includes constants and doubly logarithmic terms, and $\kappa_i = \min |d_{\min} + 2^{i-1} - d_{q,j}|$. 
\end{theorem}
\begin{proof}
This result follows by noticing that \texttt{Identify-Cover} is equivalent to \st2 with parameters $\epsilon = 2^{i-1}$ and $\gamma = 2^{i-2}$
from \cite{mason2020finding} if all distance are multiplied by $-1$. Hence, Theorem
3.1 of \cite{mason2020finding} implies the result. 
\end{proof}

Next we give a bound on the complexity of \texttt{find-smallest-in-set} which is equivalent to the bandit algorithm \texttt{Successive-Elimination} given the negatives of the distances. 

\begin{theorem}[\cite{even2002pac}, Theorem 2]\label{thm:successive_elim_complex}
Fix $\delta > 0$, a query point $\bq$, and a set $Q$. Let $\bx^\ast \in Q = \arg\min_{\bx \in Q}d(\bq, \bx)$. \texttt{find-smallest-in-set} returns $\bx^\ast$ with probability at least $1-\delta$ in at most
$$ c_2\log\left(\frac{n}{\delta}\right)\sum_{j\in Q_{i_\text{bottom}}\backslash\{\bx^\ast\}}(d(\bq, \bx_q) - d(\bq, \bx_j))^{-2}$$
calls to the distance oracle. 
\end{theorem}

\begin{lemma}[\cite{mason2020finding}, Lemma F.2]\label{lem:inv_conf_bounds}
For $\delta < 2e^{-e/2}$, $\Delta \leq 2$, 
$$t \geq \frac{4}{\Delta^2}\log\left(\frac{2}{\delta}\log_2\left(\frac{12}{\delta\Delta^2}\right)\right) \implies C_{\delta}(t) = \sqrt{\frac{4\log(\log_2(2t)/\delta)}{t}} \leq \Delta.$$
\end{lemma}

\subsection{Results about cover trees}

First we state a result from \cite{krauthgamer2004navigating}. Though this result pertains to the \texttt{Navigating-Nets} data structure, as the cover tree is based on the same set of invarainces, the same result applies to cover trees. 

\begin{lemma}[\cite{krauthgamer2004navigating}, Lemma 2.3]\label{lem:krauthgamer_height} The $i_\text{top} = O(\log_2(d_{\max})$ and $i_\text{bottom} = O(\log_2(d_{\min})$ where $d_{\max} := \max_{i,j \in \X}d_{i,j}$ and $d_{\min} := \min_{i,j \in \X}d_{i,j}$. 
\end{lemma}

The remained of the statements are from \cite{beygelzimer2006cover} for the NNS problem. The core idea of many of the above results is that by managing the uncertainty in the NNSU problem, we can use results from NNS. 
\begin{lemma}[\cite{beygelzimer2006cover}, proof of Theorem 5]\label{lem:size_of_largest_set}
For expansion constant $c$, the set $\Tilde{Q}_{i} := \{j\in Q: d_{\bq, j} \leq d(\bq, Q) + 2^{i}\}$ as defined in the proof of \texttt{Noisy-Find-Nearest} obeys
$\max_i |\Tilde{Q}_i| \leq c^5$. 
\end{lemma}

\begin{lemma}[\cite{beygelzimer2006cover}, Lemma 4.1]\label{lem:num_children_bound}
The number of children of any node of a cover tree is bounded by $c^4$. 
\end{lemma}

\begin{lemma}[\cite{beygelzimer2006cover}]\label{lem:height_bound}
The maximum explicit depth of any node in a cover tree is $\frac{\log(n)}{\log(1 + 1/c^2)} = O(c^2\log(n)$. Hence $i_\text{top} - i_\text{bottom} = O(c^2\log(n))$. 
\end{lemma}

\end{document}